\documentclass[12pt]{article}
\usepackage{amsmath} 
\usepackage{algorithm}
\usepackage{algpseudocode}
\usepackage{amsfonts}   
\usepackage{amssymb}
\usepackage{booktabs}
\usepackage{pifont}
\usepackage{amsthm}
\usepackage{bbm}
\usepackage{bm}
\usepackage{tikz}
\usetikzlibrary{arrows.meta, positioning, shapes.geometric}
\usepackage{multirow}
\usepackage{graphicx}   
\usepackage{hyperref}   
\usepackage{authblk}   
\usepackage{thmtools} 
\usepackage{physics}
\usepackage[margin=0.75in]{geometry} 
\usepackage[
  backend=biber,
  style=numeric,
  sorting=nyt,  
  giveninits=true,
  maxbibnames=9
]{biblatex}
\usepackage[normalem]{ulem}
\usepackage{mathtools, stmaryrd}

\DeclareMathOperator*{\argmax}{arg\,max}

\usepackage{xparse} \DeclarePairedDelimiterX{\Iintv}[1]{\llbracket}{\rrbracket}{\iintvargs{#1}}
\NewDocumentCommand{\iintvargs}{>{\SplitArgument{1}{,}}m}
{\iintvargsaux#1} 
\NewDocumentCommand{\iintvargsaux}{mm} {#1\mkern1.5mu,\mkern1.5mu#2}
\newenvironment{accolade}{\left\{
   \arraycolsep=1pt
        \begin{aligned}
        }
        {
        \end{aligned}
    \right.
}
\newcommand{\Var}{\mathrm{Var}}

\theoremstyle{definition}

\newtheorem{remark}{Remark}[section]

\declaretheorem[name=Proposition, numberwithin=section]{proposition}

\newtheorem{lemma}{Lemma}[section]

\title{Gradient-based Active Learning with Gaussian Processes for Global Sensitivity Analysis}
\author[1,2]{Guerlain Lambert\thanks{Email: \texttt{guerlain.lambert@ec-lyon.fr}}}
\author[1]{Céline Helbert\thanks{Email: \texttt{celine.helbert@ec-lyon.fr}}}
\author[2]{Claire Lauvernet\thanks{Email: \texttt{claire.lauvernet@inrae.fr}}}
\affil[1]{Institut Camille Jordan, CNRS UMR 5208, École Centrale de Lyon, Écully, France}
\affil[2]{INRAE, RiverLy, 69625 Villeurbanne, France}
\date{December 25, 2025} 

\begin{document}
\maketitle

\begin{abstract}
Global sensitivity analysis of complex numerical simulators is often limited by the small number of model evaluations that can be afforded. In such settings, surrogate models built from a limited set of simulations can substantially reduce the computational burden, provided that the design of computer experiments is enriched efficiently. In this context, we propose an active learning approach that, for a fixed evaluation budget, targets the most informative regions of the input space to improve sensitivity analysis accuracy. More specifically, our method builds on recent advances in active learning for sensitivity analysis (Sobol' indices and derivative-based global sensitivity measures, DGSM) that exploit derivatives obtained from a Gaussian process (GP) surrogate. By leveraging the joint posterior distribution of the GP gradient, we develop acquisition functions that better account for correlations between partial derivatives and their impact on the response surface, leading to a more comprehensive and robust methodology than existing DGSM-oriented criteria. The proposed approach is first compared to state-of-the-art methods on standard benchmark functions, and is then applied to a real environmental model of pesticide transfers.
\end{abstract}
\vfill

\newpage
\section{Introduction}
In recent decades, numerical experimentation has emerged as a compelling alternative to costly field trials for studying physical phenomena and evaluating the impact of human activities on the environment. However, numerical simulation can be computationally challenging due to detailed modeling, involving extensive computation times and careful parameterization. Moreover, when a model is based on numerous input parameters, it becomes valuable to perform global sensitivity analysis (GSA), which aims either to
reduce the dimensionality of the input space  (screening step) or to classify features by order of importance relative to the model output (ranking step) \cite{gsa, GSA_review}. Kernel-based measures such as HSIC (Hilbert-Schmidt Independence Criterion \cite{hsic}) are most commonly used to detect the subset of influential inputs whereas the Sobol indices \cite{sobol_sensitivity_1993}, based on variance decomposition, are most commonly used for ranking purposes. However, Sobol indices estimation requires a substantial number of model evaluations, which is often impractical in real-world applications.

To address these computational challenges, the use of a metamodel (or surrogate model) built from simulations of the complete model allows for significant reduction in computational burden. Nevertheless, building a surrogate model still requires a set of observations from the computational code, necessitating either space-filling designs of experiments such as low-discrepancy sequences or optimized Latin Hypercube Sampling (LHS) designs \cite{sobol1967distribution,mckay_comparison_1979,QLHS}. When a single model evaluation requires several dozen hours of computation, constructing an experimental design optimally suited to our objectives becomes crucial. For instance, as illustrated in Figure \ref{fig:doe_motiv} on a toy function, errors of metamodel-based estimation of Sobol indices (first-order and total-order) vary significantly with the size and quality of the experimental design (maximin LHS in this case), indicating the existence of optimal designs even with small sample sizes. 
Thus, sequentially enriching the design according to a specific objective function appears to be a promising methodology to find these designs. 
This is the role of active learning. 
Starting from an initial experimental design, 
it appears natural 
to add points that 
provide the most information for getting an accurate Gaussian process (GP) metamodel, and in particular for the quantities needed in sensitivity analysis.
A classical line of work considers sequential designs that aim at reducing the predictive uncertainty of the GP, for instance by minimizing an integrated mean squared error (IMSE) criterion over the input space \cite{IMSE1,IMSE2,IMSE3}. These approaches are well suited to improve the global predictive quality of the surrogate, but they do not directly target GSA indices such as Sobol' or Derivative-based Global Sensitivity Measures (DGSM) \cite{DGSM}. Some propose an acquisition function for the numerator of Sobol indices \cite{gratiet}, but this is costly to implement as there is no closed form. 
Acquisition functions have also been proposed using the framework of Active Subspaces \cite{constantine, wycoff}, which identify low-dimensional directions along which the function varies most. However, these methods primarily target dimensionality reduction rather than the precise estimation of global sensitivity measures. In this work, we focus on DGSM, which quantify the influence of each input variable through the expected squared partial derivatives and provide scalar global importance indices with a direct theoretical connection to Sobol' total indices \cite{DGSM,lamboni}. Their estimation is particularly convenient under a GP surrogate, and has been analyzed in detail in \cite{Lozzo}. This motivates the development of acquisition strategies specifically designed for improving DGSM estimates in the context of GP metamodeling. In this direction, \cite{Belakaria} introduced several DGSM-based acquisition functions, demonstrating the potential of derivative-aware active learning for sensitivity analysis. However, the authors consider GP partial derivatives as independent random variables. This assumption is unrealistic in the context of a physical model, as the correlation between partial derivatives is significant and must be considered. Furthermore, whereas numerical experiments are conducted over a wide range of test functions, the very limited number of iterations considered in these numerical tests 
does not allow for a proper assessment of the method's behavior. Finally, the proposed acquisition functions  remain mainly local or built as an aggregation of local functions. In other words, when a candidate point is considered, its influence is confined to itself and not to the whole response surface. To address these issues, we first propose acquisition functions that leverage the joint posterior distribution of the full GP gradient, thereby preserving correlations between partial derivatives. We then introduce a global acquisition criterion designed to account for the impact of a new evaluation on the GP response surface over the whole input domain. The resulting methodology is formulated in a multivariate framework that can take into account dependencies between input variables. It also includes one global acquisition function and with two local variants. After reviewing Gaussian process regression in Section \ref{section:1}, we present related work in Section \ref{section:2}. The proposed GSA-targeted active learning methods are introduced in Section \ref{section:nous} and evaluated on classical test functions. Finally, we extend the approach to dependent inputs and apply it to an environmental model assessing the effectiveness of vegetative filter strips for limiting pesticide transfer
and surface-water contamination.

\begin{figure}
    \centering
\includegraphics[width=0.66\linewidth]{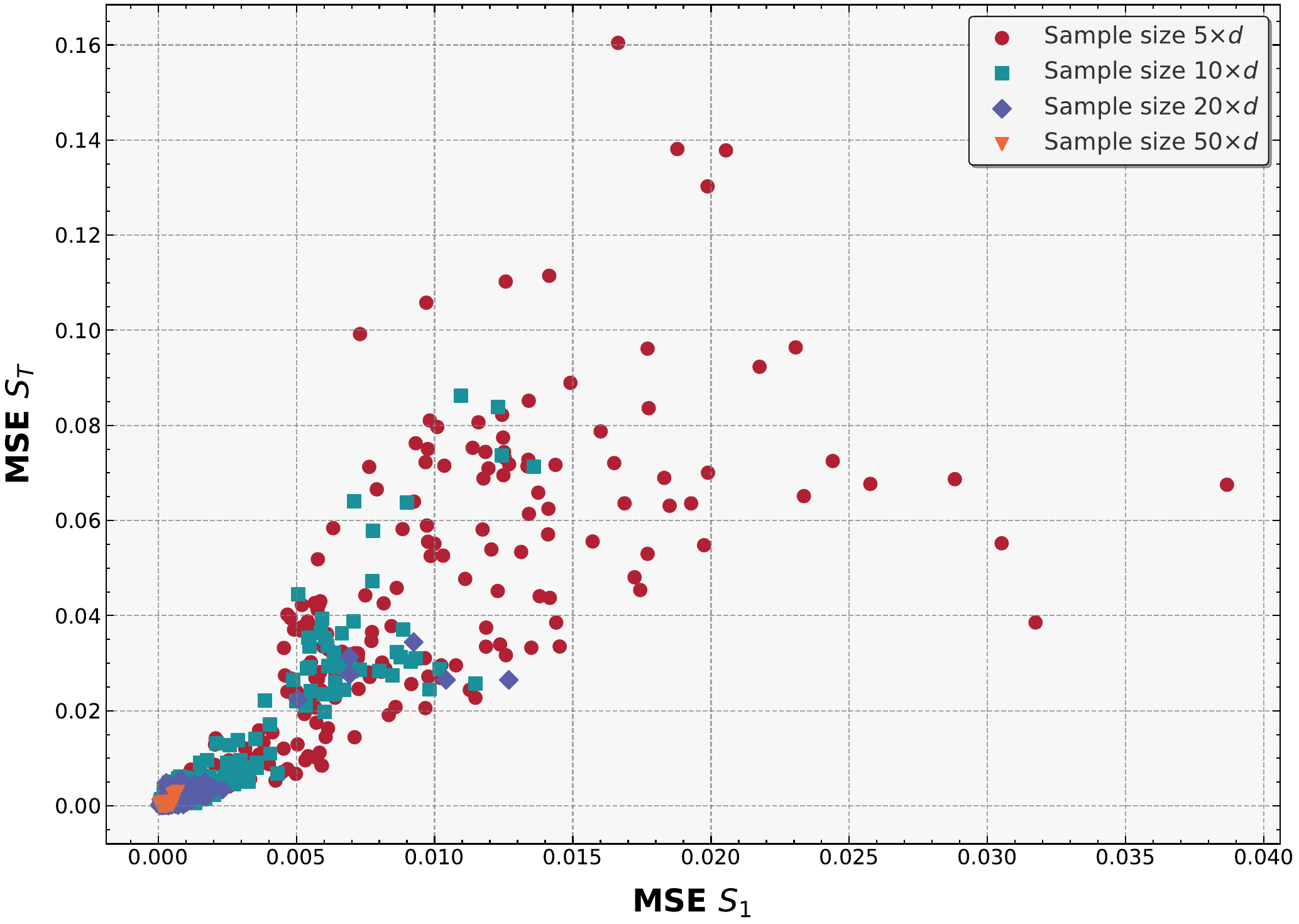}
    \caption{Aggregated MSE of first-order $S_1$ and total-order Sobol' indices $S_T$ for the $d=10$ G-Sobol function, estimated from Gaussian process surrogates trained on different designs of computer experiments. Each point corresponds to a distinct training design, and colors indicate the size of the design (from $5d$ up to $50d$). For each surrogate, Sobol' indices are computed using the same Monte Carlo sample, so differences in MSE solely reflect the impact of the training design. As expected, larger designs generally yield more accurate Sobol' estimates, although some small designs (e.g., of size $5d$) occasionally produce surrogates with comparable accuracy.}
    \label{fig:doe_motiv}
\end{figure}

\section{Background on Gaussian Processes regression}\label{section:1}
Throughout this document, we consider a black-box model $f : \mathcal{X} \to \mathbb{R}$ that is costly to evaluate. Unless otherwise specified, assume that $f$ is defined on a compact subset $\mathcal{X}$ of $\mathbb{R}^d$ and $f\in \mathrm{L^2(\mathcal{X)}}$.
\subsection{Gaussian Processes}\label{sec:GP}
Working with black-box functions can be restrictive when the evaluation cost is high. In the context of sensitivity analysis, estimating Sobol indices is computationally expensive and requires numerous calls to the computational code. We therefore need to limit our reliance on the original code and employ surrogate models to approximate the black-box function. A wide variety of methodologies may be employed in the construction of such metamodels, including, but not limited to, linear models, polynomial chaos expansions, and Gaussian process (GP) regression \cite{GSA_review, PCE_GSA, MARREL2009742}. A GP $\eta$ is characterized by:
\begin{itemize}
    \item A mean function $m : \mathbb{R}^d \to \mathbb{R}$, often assumed constant.
    \item A covariance function (or kernel) $k :\mathbb{R}^d\times \mathbb{R}^d \to \mathbb{R}$ parameterized by $\sigma^2$ (the output scale) and $(\theta_1, \dots, \theta_d)$ (the length scales). Common class of kernels are the RBF kernel and Matérn family. For example, the classical  Matérn $5/2$ kernel has the following expression, 
    \[
k_\text{Mat52}(\mathbf{x}, \mathbf{x}') = \sigma^2 \left(1 + \sqrt{5}r(\mathbf{x},\mathbf{x}') + \frac{5}{3}r(\mathbf{x},\mathbf{x}')^2\right) \exp\left(-\sqrt{5}r(\mathbf{x},\mathbf{x}')\right),
\]  
where $r(\mathbf{x},\mathbf{x}') = \sqrt{\sum_{i=1}^d\frac{(x_i - x_i')^2}{\theta_i^2}}$.
\end{itemize}

Suppose that $f$ is modeled as a realization of a GP $\eta$ with mean function $m$ and covariance function $k$. Let $\boldsymbol{\theta} = (\sigma^2, \theta_1, \dots, \theta_d)$ denote the hyperparameters of $k$, tuned for instance by maximum likelihood or cross-validation from a dataset
\[
\mathcal{D} = \left\{(\mathbf{x}_i, y_i)\right\}_{i=1}^n \subset \mathcal{X} \times \mathbb{R},
\qquad
y_i = f(\mathbf{x}_i).
\]
Set $\mathbf{X} = [\mathbf{x}_1, \dots, \mathbf{x}_n]^\top$,
$\mathbf{y} = [y_1, \dots, y_n]^\top$,
$\mathbf{m} = [m(\mathbf{x}_1), \dots, m(\mathbf{x}_n)]^\top$
and let $\mathbf{K} = \left[k(\mathbf{x}_i, \mathbf{x}_j)\right]_{i,j=1}^n$ be the covariance matrix.
Given a new point $\mathbf{x} \in \mathcal{X}$, the posterior distribution of $\eta(\mathbf{x})\mid\mathcal{D}$ is Gaussian,
\[
\eta(\mathbf{x}) \mid \mathcal{D} \sim \mathcal{N}\left(\mu(\mathbf{x}), \sigma^2(\mathbf{x})\right),
\]
with
\[
\begin{accolade}
\mu(\mathbf{x}) &= m(\mathbf{x}) 
+ \mathbf{k}(\mathbf{x})^\top \mathbf{K}^{-1} \left(\mathbf{y} - \mathbf{m}\right),\\[0.3em]
\sigma^2(\mathbf{x}) &= k(\mathbf{x}, \mathbf{x}) 
- \mathbf{k}(\mathbf{x})^\top \mathbf{K}^{-1} \mathbf{k}(\mathbf{x}),
\end{accolade}
\]
where $\mathbf{k}(\mathbf{x}) = \left[k(\mathbf{x}, \mathbf{x}_1), \dots, k(\mathbf{x}, \mathbf{x}_n)\right]^\top$
is the covariance vector between $\eta(\mathbf{x})$ and $\left(\eta(\mathbf{x}_1), \dots, \eta(\mathbf{x}_n)\right)$.

\begin{remark}
In the whole document, we work with a noise-free model, that is $y_i = f(\mathbf{x}_i)$ for all $i$. The noisy case, with observations
$y_i = f(\mathbf{x}_i) + \varepsilon_i$, $
\varepsilon_i \sim \mathcal{N}\left(0, \sigma_{\text{noise}}^2\right)$
is handled by the standard modification
$\mathbf{K} \quad\longrightarrow\quad \mathbf{K} + \sigma_{\text{noise}}^2 \mathbf{I}_n$
in all posterior formulas. All the methodologies and constructions presented below extend straightforwardly to this noisy setting.
\end{remark}

\subsection{Gradient of a Gaussian process}\label{section:gradientGP}

Let $\mathbf{X}_s = \left(\mathbf{x}^{(s)}_1, \dots, \mathbf{x}^{(s)}_N\right) \in \mathcal{X}^N$ be a set of $N$ points in the input space. We define the stacked gradient vector
\[
\nabla \eta(\mathbf{X}_s)
:= \left(
\nabla \eta\left(\mathbf{x}^{(s)}_1\right)^\top,
\dots,
\nabla \eta\left(\mathbf{x}^{(s)}_N\right)^\top
\right)^\top \in \mathbb{R}^{Nd}.
\]
A key property of GP is that applying a linear operator to a GP yields another GP (with transformed mean and covariance functions), in particular when the operator is a differential one. Indeed, if
\[
m \in \mathcal{C}^1\left(\mathcal{X}; \mathbb{R}\right)
\quad\text{and}\quad
k \in \mathcal{C}^2\left(\mathcal{X} \times \mathcal{X}; \mathbb{R}\right),
\]
then $\eta$ is differentiable almost surely, the gradient process $\nabla \eta$ is again a GP, and its posterior distribution is Gaussian. The joint posterior distribution of the gradient at $\mathbf{X}_s$ is given by
\[
\nabla \eta(\mathbf{X}_s) \mid \mathcal{D}
\sim \mathcal{N}\left(\mu_\nabla(\mathbf{X}_s), \Sigma_\nabla(\mathbf{X}_s)\right),
\]
with
\[
\begin{accolade}
\mu_\nabla(\mathbf{X}_s) &= \nabla m(\mathbf{X}_s)
+ \nabla_x k(\mathbf{X}_s, \mathbf{X}) \mathbf{K}^{-1} \left(\mathbf{y} - \mathbf{m}\right),\\[0.3em]
\Sigma_\nabla(\mathbf{X}_s) &= \nabla_x \nabla_{x'} k(\mathbf{X}_s, \mathbf{X}_s)
- \nabla_x k(\mathbf{X}_s, \mathbf{X}) \mathbf{K}^{-1} \nabla_{x'} k(\mathbf{X}, \mathbf{X}_s).
\end{accolade}
\]

We now detail the notation and structure of the derivative matrices. We define the gradient of the mean function : 
\[
\nabla m(\mathbf{X}_s)
:= \left(
\nabla m\left(\mathbf{x}^{(s)}_1\right)^\top,
\dots,
\nabla m\left(\mathbf{x}^{(s)}_N\right)^\top
\right)^\top \in \mathbb{R}^{Nd}.
\]

\subsubsection*{The matrix $\nabla_x k(\mathbf{X}_s, \mathbf{X}) \in \mathbb{R}^{Nd \times n}$}
By definition, $\nabla_x k(\mathbf{X}_s, \mathbf{X})$ is the matrix of gradients of the kernel with respect to its first argument, evaluated at pairs of gradient locations $\mathbf{x}^{(s)}_\ell$ and training points $\mathbf{x}_i$:
\[
\nabla_x k(\mathbf{X}_s, \mathbf{X})
:=
\begin{bmatrix}
\nabla_x k(\mathbf{x}^{(s)}_1, \mathbf{x}_1) & \cdots & \nabla_x k(\mathbf{x}^{(s)}_1, \mathbf{x}_n)\\
\vdots & & \vdots\\
\nabla_x k(\mathbf{x}^{(s)}_N, \mathbf{x}_1) & \cdots & \nabla_x k(\mathbf{x}^{(s)}_N, \mathbf{x}_n)
\end{bmatrix}.
\]
Each $\nabla_x k(\mathbf{x}^{(s)}_\ell, \mathbf{x}_i)$ is a column vector in $\mathbb{R}^d$, and these vectors are stacked vertically point by point: the block of rows $((\ell-1)d+1):(\ell d)$ corresponds to $\mathbf{x}^{(s)}_\ell$, with coordinates ordered by $j=1,\dots,d$. Indexing by 
$\ell \in \{1,\dots,N\}$ representing a point in the input space, $j \in \{1,\dots,d\}$ representing an input variable,
and using the linear index $r = (\ell - 1)d + j$, the entry in row $r$ and column $i$ of $\nabla_x k(\mathbf{X}_s, \mathbf{X})$ is
\[
\bigl[\nabla_x k(\mathbf{X}_s, \mathbf{X})\bigr]_{(\ell - 1)d + j, i}
= \frac{\partial}{\partial (x^{(s)}_\ell)_j}  k\left(\mathbf{x}^{(s)}_\ell, \mathbf{x}_i\right),
\quad
i = 1,\dots,n.
\]

Equivalently, the $d \times n$ block corresponding to a fixed $\ell$ (rows $(\ell-1)d+1$ to $\ell d$) gathers all derivatives of
\[
k\left(\mathbf{x}^{(s)}_\ell, \mathbf{x}_1\right), \dots, k\left(\mathbf{x}^{(s)}_\ell, \mathbf{x}_n\right)
\]
with respect to the $d$ coordinates of $\mathbf{x}^{(s)}_\ell$.

\subsubsection*{The matrix $\nabla_x \nabla_{x'} k(\mathbf{X}_s, \mathbf{X}_s) \in \mathcal{M}_{Nd}(\mathbb{R})$}
This matrix is block-defined, for $\ell, \ell' \in \{1,\dots,N\}$, the $(\ell,\ell')$ block is a $d \times d$ matrix whose $(j,j')$ entry is the mixed second derivative
\[
\left[\nabla_x \nabla_{x'} k(\mathbf{X}_s, \mathbf{X}_s)\right]_{\ell,\ell'}
=
\left[
\frac{\partial^2}{\partial (x^{(s)}_\ell)_j  \partial (x^{(s)}_{\ell'})_{j'}}
k\left(\mathbf{x}^{(s)}_\ell, \mathbf{x}^{(s)}_{\ell'}\right)
\right]_{j,j' = 1}^d.
\]
For a single point $\mathbf{x} \in \mathcal{X}$ (corresponding to $N = 1$ and $\mathbf{X}_s = (\mathbf{x})$), the gradient is a $d$-dimensional Gaussian vector
\[
\nabla \eta(\mathbf{x}) \mid \mathcal{D}
\sim \mathcal{N}\left(\mu_\nabla(\mathbf{x}), \Sigma_\nabla(\mathbf{x})\right),
\]
with
\[
\begin{accolade}
\mu_\nabla(\mathbf{x}) &= \nabla m(\mathbf{x})
+ \nabla_x k(\mathbf{x}, \mathbf{X}) \mathbf{K}^{-1} \left(\mathbf{y} - \mathbf{m}\right),\\[0.3em]
\Sigma_\nabla(\mathbf{x}) &= \nabla_x \nabla_{x'} k(\mathbf{x}, \mathbf{x})
- \nabla_x k(\mathbf{x}, \mathbf{X}) \mathbf{K}^{-1} \nabla_{x'} k(\mathbf{X}, \mathbf{x}).
\end{accolade}
\]

In particular, for any fixed coordinate $j \in \{1,\dots,d\}$, the partial derivative process
$\mathbf{x} \longmapsto \frac{\partial \eta(\mathbf{x})}{\partial x_j}$
is a GP. Its posterior marginal at $\mathbf{x}$ is
\[
\frac{\partial \eta(\mathbf{x})}{\partial x_j} \,\bigg|\, \mathcal{D}
\sim \mathcal{N}\left(\mu_j'(\mathbf{x}), \Sigma_j'(\mathbf{x})\right),
\]
where
$\mu_j'(\mathbf{x}) := \left[\mu_\nabla(\mathbf{x})\right]_j$ and 
$\Sigma_j'(\mathbf{x}) := \left[\Sigma_\nabla(\mathbf{x})\right]_{jj}$.\\

When the coordinate $j$ is fixed and clear from the context, we simply write $\mu'(\mathbf{x})$ and $\Sigma'(\mathbf{x})$. With ARD Matérn–$5/2$ and RBF kernels, the gradient and Hessian of $k$ admit closed-form expressions. The explicit formulas for Matérn kernels are given in Appendix \ref{appendix:kernel}. A visual representation of GP regression and its gradient posterior distribution is given in Figure \ref{fig:GPgrad}.

\begin{figure}
    \centering
    \includegraphics[width=0.66\linewidth]{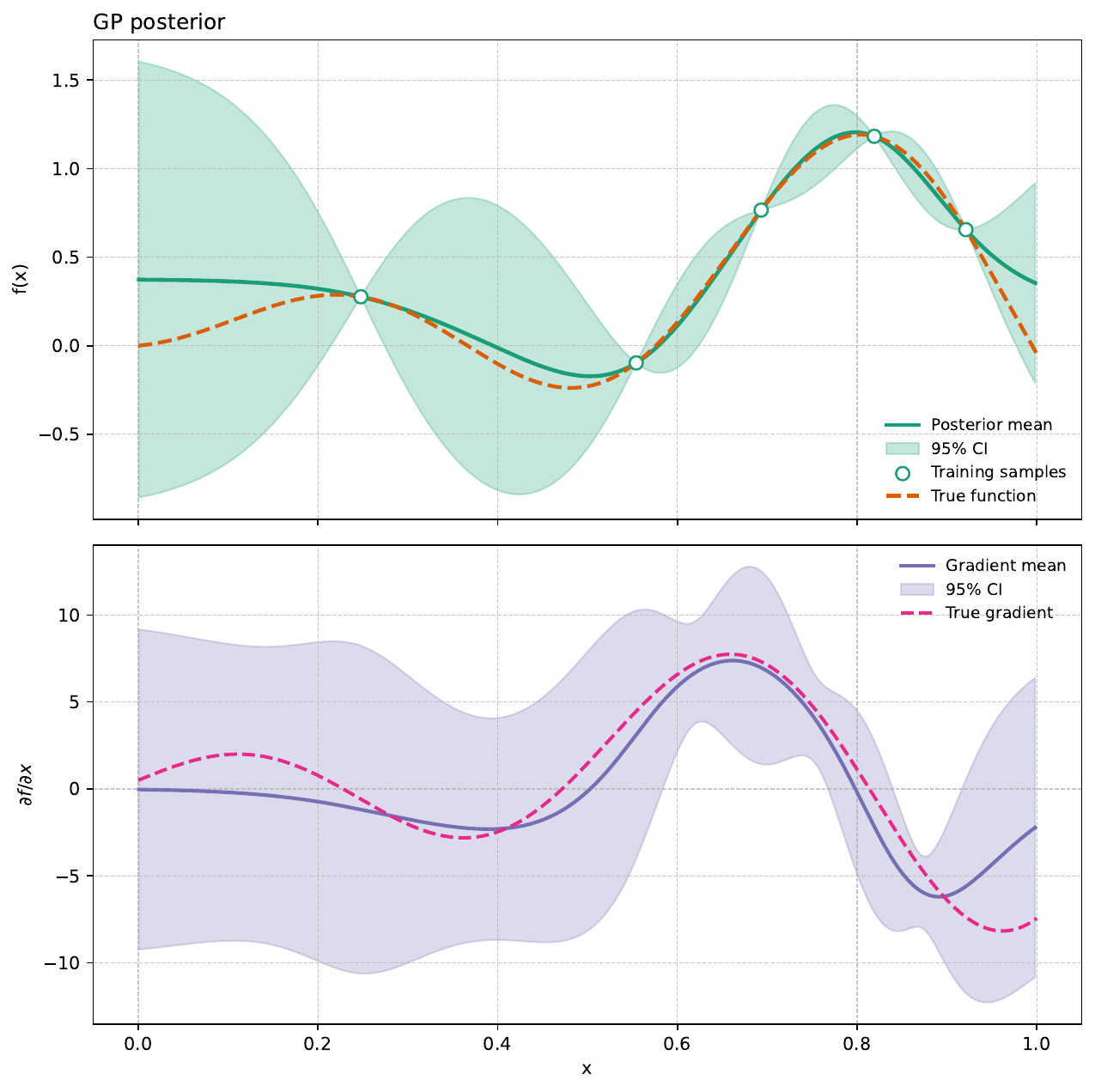}
    \caption{Gaussian process regression on a 1D toy function. (Top) Posterior distribution of $\eta$. (Bottom) Posterior distribution of $\nabla\eta$.}
    \label{fig:GPgrad}
\end{figure}

\section{Derivative-based active learning}\label{section:2}
We work within a numerical modeling framework in which environmental inputs are represented by random variables (or vectors), and our objective is to perform a sensitivity analysis to characterize how uncertainties in the inputs propagate to the model output. A widely used family of sensitivity indices in this context is the Sobol’ variance-based indices \cite{sobol_sensitivity_1993,gsa}. For any subset of input variables $\bm{u} \subset\{1,\dots, d\}$, the first-order Sobol' index is defined as 
\[S_{\bm{u}} =  \frac{\mathrm{Var}\left(\mathbb{E}\left[f(X)\mid X_{\bm{u}}\right]\right)}{\mathrm{Var}\left(f(X)\right)}\]
and quantifies the part of the output variance caused only by the variables in $\bm{u}$. The total-order Sobol' index is given by
\[S_{\bm{u}}^\top = 1 - S_{{}_{\sim}\bm{u}}\]
where ${}_{\sim}\bm{u}$ denotes the complement of $\bm{u}$. This index measures the overall contribution of the group $\bm{u}$, including all the interaction effects with the other inputs. In the particular case of a single variable $k$, we write $S_{k}$ and $S_k^\top$ as shorthand for $S_{\bm{u}}$ and $S_{\bm{u}}^\top$ with $ \bm{u} =\{k\}$.\\
Although Sobol' indices provide an interpretable variance decomposition, their estimation may require numerous model evaluations, especially for high-dimensional models. For this reason, we consider an alternative class : the Derivative-based Global Sensitivity Measures (DGSM), introduced by \cite{DGSM}. In contrast to variance-based methods, DGSMs rely on the derivatives of the model with respect to the inputs.
\subsection{Derivative-based Global sensitivity Measure}\label{sec: derivative}
From now on, to facilitate subsequent developments, we assume that $f \in H^1(\mathcal{X})$, where the Sobolev space $H^1(\mathcal{X})$ is defined by,
\[
H^1(\mathcal{X})
= \left\{ g \in L^2(\mathcal{X}) \mid \nabla g \in L^2(\mathcal{X}) \right\}.
\]
This assumption is compatible with the physics-based models considered in this work. Let $(\Omega,\mathcal{A}, \mathbb{P})$ a probability space and  $X : (\Omega, \mathcal{A, \mathbb{P)}} \to \mathbb{R^d}$ be a random vector with probability measure $\mu_X$ supported on $\mathcal{X}$. For each marginal $X_k$ of $X$, we note $\mu_k$ its probability measure.
The $k$-th DGSM is defined by
\[
D_k
= \mathbb{E}\left[ \left( \partial_k f(X) \right)^2 \right]
= \int_{\mathcal{X}} \left( \partial_k f(x) \right)^2  \mu_X(\mathrm{d}x).
\]
Such indices are commonly used as screening indices \cite{basicstrends} but cannot be used to rank inputs. However, the interesting property stated below allows to link DGSMs to variance-based indices. Suppose that the inputs variables $X_1, \dots, X_d$ are independents, and $\mu_k$ satisfies a (1-dimensional) Poincaré inequality :
\[\int g(x)^2\mu_k(\mathrm{d}x) \leq C(\mu_k)\int g'(x)\mu_k(\mathrm{d}x)\]
for all \textit{sufficiently} regular centered functions $g$. The best constant $C(\mu_k)$ achieving this inequality is noted $C_k$. Suppose also that $\mu_k$ is a Boltzmann measure, i.e.,
\[
  \mu_k(\mathrm{d}x)
  = Z_k^{-1}\mathrm{e}^{-V_k(x)}\mathrm{d}x,
\]
for some potential $V_k:\mathbb{R}\to\mathbb{R}$ with $V_k\in C^2$ and $\mathrm{e}^{-V_k}\in L^1(\mathbb{R})$. Then, 
\begin{align}\label{DGSMsobol}
    S_k^\top \leq \frac{C_k D_k}{\mathrm{Var}(f(X))}
\end{align}
where
\[C_k= 4\left[\sup_{x\in\mathbb{R}}\frac{\min(F_k(x), 1-F_k(x))}{\rho_k(x)}\right]^2\]
with, $F_k$ (resp. $\rho_k$) the c.d.f (resp. p.d.f) associated to $X_k$ \cite{lamboni}. Note that any log-concave probability measure is a Boltzmann measure. Then, the constant $C_k$ is known in closed-form for popular probability distribution \cite{iooss_dgsm}. The inequality in \ref{DGSMsobol} is an equality if $f$ is linear. One can note that the DGSMs are a particular case of the Active Subspaces (AS) \cite{constantine}. Indeed, the DGSMs are the diagonal elements of the AS matrix. \\

The strong link between derivative-based and variance-based sensitivity analysis with respect to the Equation \ref{DGSMsobol} motivates the use of DGSM to enrich the metamodel. Additionally, the GP benefits easy access to the posterior distribution of the partial derivatives. Hence, it is simple to build acquisition functions upon the squared derivatives of the surrogate. Furthermore, estimating indices is easy in the context of GP thanks to developments by \cite{Lozzo}, where estimation is performed either by a plug-in of the GP mean or by using the GP paths.

\subsection{Single-loop active learning strategy}
The classical Bayesian active learning strategy is described in Algorithm \ref{alg:ALalg}. The idea is, for a total budget $T$ to sequentially add points given a certain acquisition function $\alpha : \mathcal{X} \to\mathbb{R}$ that is encoding the most informative regions given our objective. Hence, it is crucial to design a relevant function based on the characteristics of the model.

\begin{algorithm}
\caption{Active Learning loop}
\label{alg:ALalg}
\begin{algorithmic}[1]
    \State \textbf{Input:} Initial DoE $\mathcal{D}_0 = \{(\mathbf{x}_i, y_i)_{i=1,\dots, N}\}$, a surrogate $\mathcal{GP}_0$ and a total budget $T$.
    \For{$t = 1,\dots, T$}
        \State Train a GP using $\mathcal{D}_{t-1}$
        \State Find the candidate $\mathbf{x}_*$ such that:
        \Statex \[
        \mathbf{x}_* \in \argmax_{\mathbf{x} \in \mathcal{X} \subset \mathbb{R}^d}
        \alpha\left(\mathbf{x}, \mathcal{GP}(\mathcal{D}_{t-1})\right)
        \]
        \State Compute $y_* = f(\mathbf{x}_*)$
        \State $\mathcal{D}_{t} = \mathcal{D}_{t-1}\cup\{(\mathbf{x}_*, y_*)\}$
    \EndFor
\end{algorithmic}
\end{algorithm}
\subsection{Derivative-based acquisition functions}\label{sec:eux}
Recent developments from \cite{Belakaria} introduce several acquisition functions that target the squared individual partial derivatives, whose distribution is analytically tractable in the GP framework. Specifically, the posterior distribution of $\partial_i \eta(\mathbf{x})\mid \mathcal{D}$ follows a Gaussian distribution, and it is well-established that when $W \sim \mathcal{N}(\mu, \sigma^2)$, then $W^2 \sim \sigma^2\chi^2_1\left(\frac{\mu^2}{\sigma^2}\right)$. Denoting $\sigma'_{j,\text{sq}}(\mathbf{x})^2$ as the variance of $\left(\frac{\partial \eta(\mathbf{x})}{\partial x_j}\right)^2\mid \mathcal{D}$, we have, for a non-central $\chi^2$ distribution :
\begin{equation}\label{eq:vareux}
\sigma'_{j,\text{sq}}(\mathbf{x})^2 = 2\left(\Sigma'_j(\mathbf{x})^4 + 2\mu'_j(\mathbf{x})^2\Sigma_j'(\mathbf{x})^2\right).
\end{equation}

The proposed acquisition functions identify candidate points $\mathbf{x}$ where either the prediction variance or the variance reduction is maximized:
\[\alpha_{\text{PartialMaxVar}}(\mathbf{x}) = \sum_{j=1}^d\sigma'_{j,sq}(\mathbf{x})^2\quad\text{and}\quad \alpha_{\text{PartialRedVar}}(\mathbf{x}) = \sum_{j=1}^d \sigma'_{j,sq}(\mathbf{x})^2 - \mathbb{E}_y\left[\sigma'_{j,sq,\ell}(\mathbf{x})^2\right]\]
where $\sigma'_{j,\text{sq},\ell}(\mathbf{x})^2$ represents the one-step look-ahead variance. The look-ahead distribution, also called fantasy, is obtained by sampling the posterior distribution $N_f$ times at a candidate point $\mathbf{x}$, yielding a collection ${(\mathbf{x}, y_i)}_{i=1\ldots N_f}$. Subsequently, $\eta$ is conditioned independently to obtain a collection of GPs $(\eta(\mathbf{x}) \mid  \mathcal{D}\cup\{(\mathbf{x}, y_i)\})_{i=1\ldots N_f}$. The concept of fantasy modeling is reminded in Section \ref{sec:fanty}. The one-step look-ahead distribution of the partial derivatives of the GP can then be computed, with the sole dependence on observations manifesting through the look-ahead mean. \\

Despite their intuitive appeal and practical utility, these acquisition functions exhibit several limitations. First, following the methodology described by \cite{Belakaria}, the look-ahead mean $\mu_{j,\ell}'$ is approximated by $\mu_j'$, yielding a look-ahead variance for squared derivatives that is independent of $y$. Second, the correlation between derivatives, which is inherent to the GP construction, is only partially captured in this formulation. Most importantly, these criteria remain local: they focus on local variance effects and do not account for the impact of new evaluations on the global variance of the surrogate model. In the following section, we introduce an alternative construction that explicitly addresses each of these three limitations.

\section{Gradient-based active learning}\label{section:nous}
In this section, we propose a new acquisition strategy specifically designed to overcome the three limitations discussed above. First, we clarify the look-ahead aspect by using fantasy GP, which offer a consistent approach to handling dependence on future observations. Next, we introduce two new classes of acquisition function. The first extends the previously defined local criteria by explicitly incorporating gradient information and the joint distribution of derivatives, thereby better capturing their correlation structure. The second class is a global variant designed to account for the impact of new evaluations on the overall uncertainty of the surrogate model, rather than only local variance. Finally, we present a block-based approach to facilitate the practical computation of these acquisition functions when the input dimension is large.
\subsection{Vector-valued active learning}
\subsubsection*{Look-ahead approximation : GP fantasy}\label{sec:fanty}

Many variance-reduction acquisition functions rely on a look-ahead criterion, in which the effect of a potential new evaluation at $\mathbf{x} \in \mathcal{X}$ is assessed through an expectation of the form
\[
\mathbb{E}_{y}\left[\operatorname{Var}\left(Z \mid \mathcal{D} \cup \{(\mathbf{x}, y)\}\right)\right],
\]
where $Z$ denotes some quantity of interest derived from the GP (e.g., a derivative, an integral, or a functional of the surrogate). Such criteria are typical of the Stepwise Uncertainty Reduction (SUR) paradigm \cite{bect_sequential_2012}. This look-ahead variance explicitly depends on the yet-unobserved value $y = f(\mathbf{x})$, which is not available in practice. A natural way to handle this dependence is to interpret $y$ as a random variable distributed according to the current posterior of the GP at $\mathbf{x}$, and to approximate the above expectation by averaging over plausible future observations.

This is precisely the role of fantasy GP. For a candidate point $\mathbf{x}$, we draw hypothetical observations $y^{(1)}, \dots, y^{(N_f)}$ from the predictive distribution of $\eta(\mathbf{x}) \mid \mathcal{D}$ and, for each draw $y^{(m)}$, consider the corresponding fantasy dataset $\mathcal{D} \cup \{(\mathbf{x}, y^{(m)})\}$. Each fantasy dataset defines a fantasy GP, with its own posterior mean and covariance. By computing the variance of $Z$ under these fantasy GPs and averaging over $m = 1,\dots,N_f$, we obtain a Monte Carlo approximation of the look-ahead criterion above, without ever observing $y$ in reality. The related idea of conditioning on fictitious “fantasy” observations to implement look-ahead and batch acquisition functions appears in early parallel Bayesian optimization work \cite{Snoek2012PracticalBO, Ginsbourger2008,Ginsbourger2010} and is now widely known under the name of “fantasy models” in modern GP toolboxes (e.g., BoTorch / GPyTorch \cite{gpytorch, botorch}). An illustrative example of the fantasy machinery is given in Figure \ref{fig:GPfant}. In the following, we exploit this construction to obtain tractable and flexible look-ahead acquisition functions.
\begin{figure}
    \centering
    \includegraphics[width=0.66\textwidth]{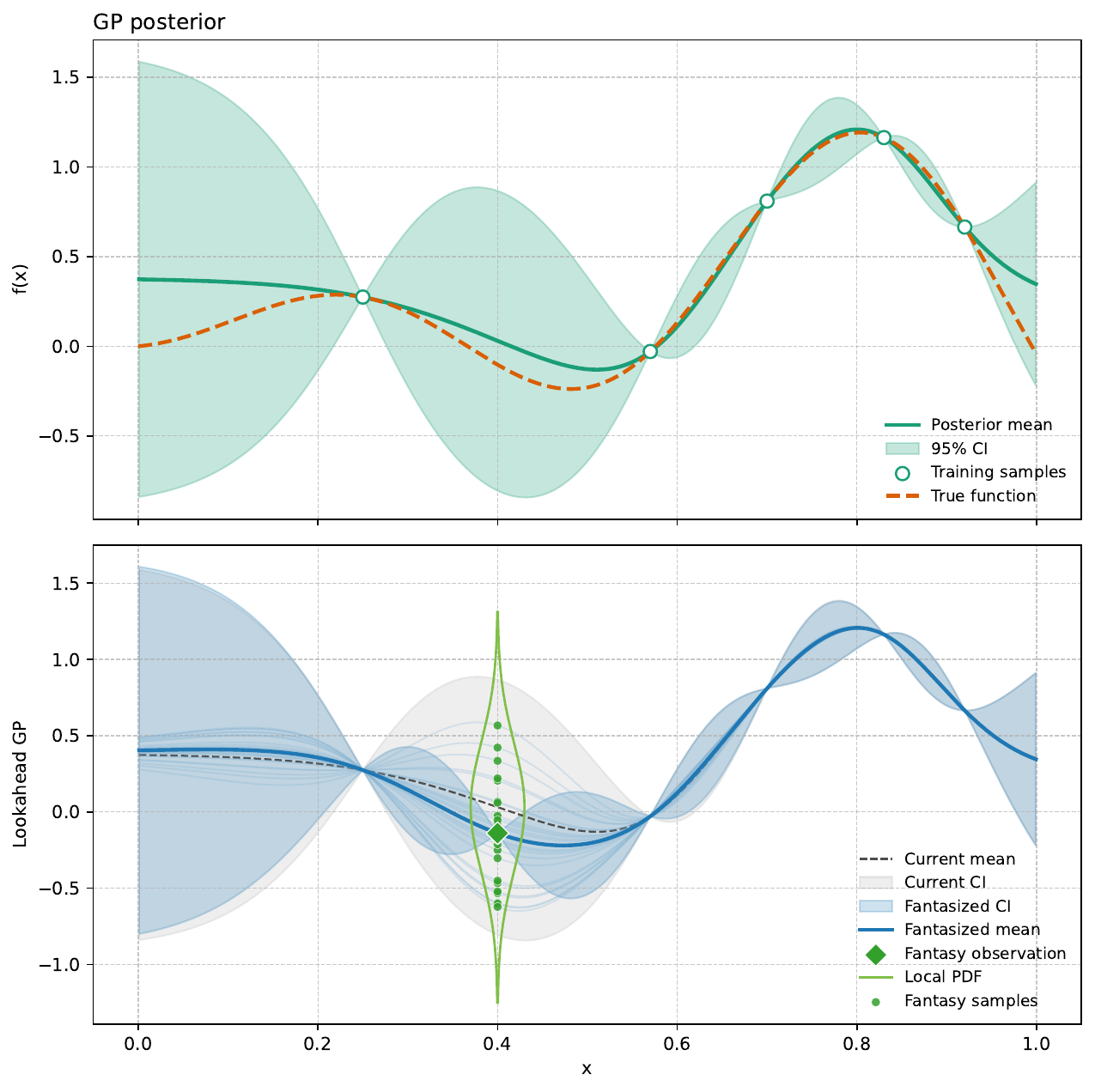}
    \caption{Gaussian process posterior and look-ahead visualization. Current posterior (grey) and a set of fantasy posteriors obtained by conditioning on hypothetical evaluations at the candidate locations $X_{\mathrm{cand}}$. 
    All fantasy trajectories are shown as thin lines, while one randomly selected fantasy path is highlighted in bold and the corresponding fantasy observation at $X_{\mathrm{cand}}$ is indicated by the marker.
    }
    \label{fig:GPfant}
\end{figure}
\subsubsection*{Gradient-based local criteria} 
To overcome the limitations of previous acquisition functions, we propose several local criteria and a global variance reduction criterion. In other words, the objective is to identify the candidate point $\mathbf{x}$ that allows the global variance to be reduced as much as possible, rather than optimizing only a local reduction. The methods in Section \ref{sec:eux} implicitly assume that:
\[\forall (j,j') \in \Iintv{1,d}^2,\quad\mathrm{Cov}\left(\left(\frac{\partial \eta(\mathbf{x})}{\partial x_j}\right)^2, \left(\frac{\partial \eta(\mathbf{x})}{\partial x_{j'}}\right)^2\right)  = 0\]
This is not realistic. Consider a candidate $\mathbf{x}\in\mathbb{R}^d$. We build the posterior distribution of the GP's gradient $Z_{\mathbf{x}}:= \nabla\eta(\mathbf{x})\mid\mathcal{D}$ known in closed form from Section \ref{section:gradientGP}. From this formulation, we can derive some local version of Maximum Variance \texttt{GradMaxVar} and Variance Reduction \texttt{GradRedVar} with 
\[\alpha_\text{GradMaxVar}(\mathbf{x}) = \mathrm{Var}\left(Z_{\mathbf{x}}^\top Z_{\mathbf{x}}\right).
\]
and,
\[\alpha_\text{GradVarRed}(\mathbf{x}) = \mathrm{Var}\left(Z_{\mathbf{x}}^\top Z_{\mathbf{x}}\right) - \mathbb{E}_y\Big[\mathrm{Var}\Big(Z_{\mathbf{x}}^\top Z_{\mathbf{x}} \mid \mathcal{D}\cup\{(\mathbf{x}, y_i)\}_{i=1,\dots,N_f}\Big)\Big]\]
We work directly with the gradient squared norm. This means that we consider the variance of the sum of the DGSMs, thus preserving all dependencies between partial derivatives. However, while taking into account the correlation structure between derivatives this criterion remains local, it measures an influence pointwise.

\subsubsection*{A global gradient-based variance reduction criterion}  
Let \(\mathbf{X_s} = (\mathbf{x}_1^*, \dots, \mathbf{x}_N^*) \in \mathcal{X}^N\) a set of random points drawn from the input distribution. We build the posterior distribution of the GP's gradient $Z_{\bm{X}_s} := \nabla \eta(\bm{X}_s) \mid \mathcal{D}\sim \mathcal{N}(\mu_\nabla, \Sigma_\nabla)$ using Section \ref{section:gradientGP}. Hence, we aim to identify the point $ \mathbf{x}$ which maximizes the variance reduction of $\Vert Z_{\bm{X}_s} \Vert_2^2 = Z_{\bm{X}_s}^\top Z_{\bm{X}_s}$. While doing so, the overall covariance structure is taken into account, i.e., covariances between partial derivatives and covariances between points in $\mathbf{X_s}$. The \texttt{GlobalGradVarRed} acquisition function is defined as: 
\[
\alpha(\mathbf{x}) = \mathrm{Var}(Z_{\bm{X}_s}^\top Z_{\bm{X}_s}) - \mathbb{E}_y\Big[\mathrm{Var}\Big(Z_{\bm{X}_s}^\top Z_{\bm{X}_s} \mid \{\mathcal{D} \cup (\mathbf{x}, y_i)\}_{i=1,\dots,N_f}\Big)\Big].
\]
As $Z^\top Z$ is a quadratic form of a Gaussian vector its variance is given by :
\begin{equation}
\label{eq:varglobal}
\mathrm{Var}(Z_{\bm{X}_s}^\top Z_{\bm{X}_s}) = 2 \mathrm{Tr}(\Sigma_\nabla^2) + 4 \mu_\nabla^\top \Sigma_\nabla\mu_\nabla
\end{equation}
As well as being global in the sense of being related to the gradient, \texttt{GlobalGradVarRed} is global in another way. Indeed, for a given candidate point $\mathbf{x}$, we examine its impact on the entire GP response surface using one-step-ahead conditioning and evaluation on a set of points $\mathbf{X_s}$, which provides a more accurate indication of the point's impact on the entire model. Moreover, this acquisition function can deal with both independent and dependent vector of inputs. The only difficulty lies in the optimization step in order to obtain candidates that are realistic and belong to the support of the input's joint distribution.

\subsection{Chunked evaluation of the global acquisition function}
As the dimension $d$ increases and the set of points $\mathbf{X}_s$ becomes finer,
the memory burden in the acquisition loop grows significantly. In practice,
$N \approx 50d$, so the covariance matrices
$\Sigma_\nabla \in \mathcal{S}^{++}_{Nd}(\mathbb{R})$ are large and costly to handle.
In addition, the look-ahead distribution is approximated using $N_f$ GP fantasies,
which introduces an extra factor in computational cost. Therefore, it is essential to
work with an approximate version of the global posterior distribution of the gradient.\\

The main idea is to approximate $\Sigma_\nabla$ by a block-diagonal matrix, where each
block corresponds to the covariance matrix evaluated over a batch of points from
$\mathbf{X}_s$. This is motivated by the fact that, for standard stationary kernels,
nearby points exhibit strong correlations while distant points are only weakly
correlated. To do so, we use a chunked evaluation of
\texttt{GlobalGradVarRed} based on KMeans clustering. Similar ideas appear
in sparse and local GP methods such as Vecchia and block-Vecchia approximations \cite{Vecchia,blockVecchia} and adapted here. In contrast, the developed approximation is specifically tailored to preserving the global quadratic
functional $\| \nabla \eta(\mathbf{X}_s)\|_2^2$ rather than the full marginal
likelihood.\\

Let denote $C$ the number of clusters.
We partition $\mathbf{X}_s$ into clusters
$\{C_i\}_{i=1}^{C}$ using KMeans on the input space and denote $n_i = \mathrm{Card}(C_i)$.
Collect the gradient components associated with a cluster $i$ into a vector
$Z^{(i)} \in \mathbb{R}^{n_i d}$, so that (up to a permutation)
$Z_{\mathbf{X}_s} = (Z^{(1)},\dots,Z^{(C)})$.
The exact covariance matrix then has a block structure
$\Sigma_\nabla = [\Sigma_\nabla^{(ij)}]_{1\le i,j\le C}$,
where $\Sigma_\nabla^{(ii)} \in \mathbb{R}^{n_i d \times n_i d}$ collects all
within-cluster covariances and $\Sigma_\nabla^{(ij)}$ the cross-cluster covariances.
We approximate $\Sigma_\nabla$ by the block-diagonal matrix,
\[
  \widetilde{\Sigma}_\nabla
  :=
  \mathrm{diag}\bigl(\Sigma_\nabla^{(11)},\dots,\Sigma_\nabla^{(CC)}\bigr),
\]
i.e., we retain all within-cluster covariances and set the cross-cluster blocks to zero.
The approximate mean is
$\widetilde{\mu}_\nabla = (\mu_\nabla^{(1)},\dots,\mu_\nabla^{(C)})$, the concatenation
of the per-cluster means. A visual representation of the chunked approximation is given in Figure \ref{fig:chunkcov}. For readability, entries are re-ordered so that gradient components associated with the same cluster appear contiguously, making the retained within-cluster blocks visible.

\begin{figure}
    \centering
    \includegraphics[width=\linewidth]{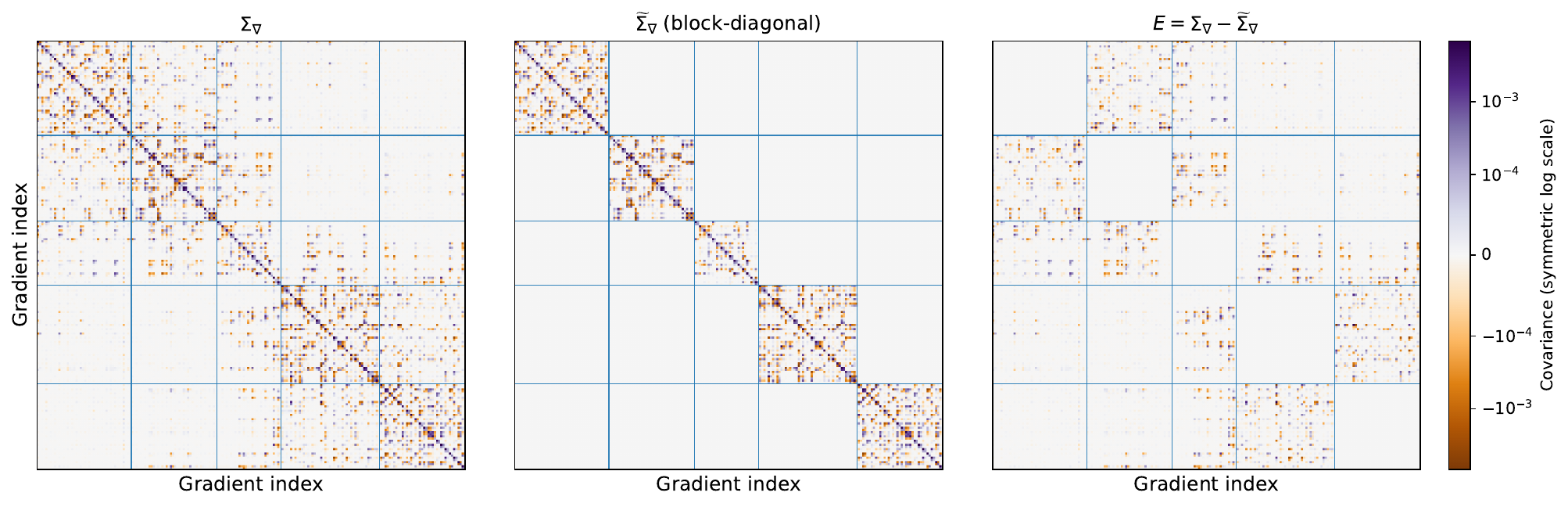}
\caption{Heatmaps of the joint posterior covariance of the GP gradient at the sampled locations $\mathbf{X}_s$: the exact covariance $\Sigma_\nabla$, its block-diagonal approximation $\widetilde{\Sigma}_\nabla$ obtained by retaining only within-cluster blocks (classical KMeans on $\mathbf X_s$), and the residual $E=\Sigma_\nabla-\widetilde{\Sigma}_\nabla$. The horizontal and vertical axes index the same flattened gradient vector $Z_{\mathbf X_s}=[\nabla\eta(x_1^*)^\top,\ldots,\nabla\eta(x_N^*)^\top]^\top\in\mathbb R^{Nd}$, so each pixel corresponds to the covariance between two partial derivatives (point and coordinate). The color scale is a symmetric logarithmic normalization.}

    \label{fig:chunkcov}
\end{figure}

Using the block-diagonal structure of $\widetilde{\Sigma}_\nabla$ in
\eqref{eq:varglobal}, the variance of the quadratic form is approximated by
\[
  \widetilde{\Var}\left(Z_{\mathbf{X}_s}^\top Z_{\mathbf{X}_s}\right)
  =
  2\Tr\bigl(\widetilde{\Sigma}_\nabla^2\bigr)
  + 4\widetilde{\mu}_\nabla^\top \widetilde{\Sigma}_\nabla \widetilde{\mu}_\nabla
  =
  \sum_{i=1}^C
  \left(
    2\Tr\bigl((\Sigma_\nabla^{(ii)})^2\bigr)
    + 4\mu_\nabla^{(i)\top} \Sigma_\nabla^{(ii)} \mu_\nabla^{(i)}
  \right).
\]
This chunked variance, which uses the full covariance structure within each cluster but
ignores cross-cluster interactions, defines the approximated acquisition function
\texttt{GlobalGradVarRedKmeans}. The chunk sizes $n_1, ..., n_C$ controls the trade-off between
accuracy and computational / memory cost: $n_1=N$ corresponds to the exact variance,
while small $n_i$'s lead to stronger sparsification.

\begin{proposition}
  \label{prop:chunked}
  Let $E = \Sigma_\nabla - \widetilde{\Sigma}_\nabla$ be the error matrix. At fixed
  clustering, the absolute error between the exact and chunked variances satisfies
  \[
    \bigl| V - \widetilde{V}\bigr|
    \le
    2 \|E\|_F^2
    + 4\|\mu_\nabla\|_2^2 \|E\|_2,
  \]
  where $V = \Var(Z_{\mathbf{X}_s}^\top Z_{\mathbf{X}_s})$,
  $\widetilde{V} = \widetilde{\Var}(Z_{\mathbf{X}_s}^\top Z_{\mathbf{X}_s})$,
  $\|\cdot\|_F$ is the Frobenius norm and $\|\cdot\|_2$ the spectral norm.
  Moreover, for the Matérn-$5/2$ kernel with output variance $\sigma^2$ and
  length-scales $\theta_1,\dots,\theta_d$, define
  \[
    A = \frac{5\sigma^2}{3}, \quad
    \theta_{\text{min}} = \min_{1\le i\le d} \theta_i, \quad
    L^2 = \sum_{i=1}^d \theta_i^{-4}.
  \]
  For $i\neq j$, let
  \[
    \Delta_{ij} = \min_{x\in C_i,x'\in C_j} r(\mathbf{x},\mathbf{x}'),
    \qquad
    \Delta = \min_{i\neq j} \Delta_{ij},
  \]
  Set
  \[
    h(r)
    =
    2A^2 \exp(-2\sqrt{5}r)
    \biggl(
      25\frac{r^4}{\theta_{\text{min}}^4}
      + (1+\sqrt{5}r)^2 L^2
    \biggr),\quad r\ge0.
  \]
  Then, for clusters of sizes $n_1,\dots,n_C$ with $N = \sum_{i=1}^C n_i$, we have
  \begin{align}
    \label{eq:bound_E_F_main}
    \|E\|_F^2
    &\le
    \Bigl(
      N^2 - \sum_{i=1}^C n_i^2
    \Bigr) h(\Delta),\\[0.4em]
    \label{eq:bound_E_2_main}
    \|E\|_2
    &\le
    B\sqrt{h(\Delta)},
  \end{align}
  where
  \[
    B =
    \begin{cases}
      N\bigl(1-\tfrac{1}{C}\bigr),
      &\text{for balanced clusters } (n_i = n_j),\\[0.3em]
      \dfrac{N + (C-2)n_{\max}}{2},
      &\text{for imbalanced clusters, with } n_{\max} = \max_i n_i.
    \end{cases}
  \]
\end{proposition}

The proof is given in Appendix \ref{appendix:proof}. Note that for imbalanced
clusters the term $\sum_{i=1}^C n_i^2$ is smaller, and the factor
$N^2 - \sum_i n_i^2$ in \eqref{eq:bound_E_F_main} is larger, so balanced clustering (achieved via constrained KMeans \cite{kmeansconstrained, kmeansconstrained_software}) reduces the bound. If the points $\mathbf{X}_s$ are generated from a Sobol sequence in $[0,1]^d$, the minimum pairwise distance typically scales as $N^{-1/d}$  \cite{niederreiter1992random, SOBOL200780}. This suggests that, in a worst-case sense, one can expect $\Delta$ to be of order $N^{-1/d}$. In practice, we compute $\Delta$ directly
from the clustered dataset.

\paragraph{Rule of thumb for the chunk sizes.}
From Proposition \ref{prop:chunked}, at fixed clustering,  the bound depends on the chunk sizes $n_i$ and the inter-cluster separation $\Delta$: increasing the number of clusters $C$ (hence reducing $n_i$) generally tightens the bound, at the price of higher computational overhead. If the exact variance $V$ is available, we select $C$ by balancing the memory budget and a target relative tolerance on $\vert V-\widetilde V\vert/V$. When $V$ cannot be computed, we choose the largest chunk sizes compatible with the memory constraint, compute $\Delta$ from the resulting clustering, and adjust $C$ accordingly. In practice, since $\Delta$ varies with the clustering itself, the error is not a deterministic function of $C$ alone and this procedure should be viewed as a heuristic guideline rather than an automatic rule.

\section{Numerical experiments}\label{section:4}
The existing acquisition methods presented in Section \ref{sec:eux}, as well as the new methods  introduced and described in Section \ref{section:nous}, are tested using several test cases and a real-life model of vegetative filter strip management for pesticide filtration on agricultural land. The experimental context is as follows:
\begin{itemize}
    \item The initial design  consists of $5d$ points generated by a Sobol sequence.
    \item A total budget of $10d$ points is allocated according to the acquisition methods. Thus, the final DoE comprises $15d$ points.
    \item Each method is repeated independently 30 times.
    \item To assess the performance of the metamodels, we track at each iteration $t$ of active learning the RMSE over input dimensions for both DGSMs and total Sobol' indices. Denoting by $D_k$  the reference DGSM and  $S_{k}^T$ the total Sobol' index of input $k$, and by $\widehat{D}_k^{(t)}$ and $\widehat{S}_{k}^{T,(t)}$ their GP-based estimates at iteration $t$, we define
\[
\mathrm{RMSE}_{\mathrm{DGSM}}^{(t)}
=
\left(
\frac{1}{d} \sum_{k=1}^d
\bigl(\widehat{D}_k^{(t)} - D_k\bigr)^2
\right)^{1/2},
\qquad
\mathrm{RMSE}_{\mathrm{Sobol}}^{(t)}
=
\left(
\frac{1}{d} \sum_{k=1}^d
\bigl(\widehat{S}_{k}^{T,(t)} - S_{k}^T\bigr)^2
\right)^{1/2}.
\]
We will also consider $Q^2$, which expresses the proportion of variance explained by the metamodel on a test set that was not used for training : 
    \[Q^2 = 1 - \frac{\sum_{i=1}^{n_{\text{test}}}(y_i - \hat{y}_i)^2}{\sum_{i=1}^{n_{\text{test}}}(y_i - \bar{y})^2}\]
where $X_\text{test} = (\mathbf{x}_1, \dots, \mathbf{x}_{n_\text{test}})\in\mathbb{R}^{n_{\text{test}}d}$ with $y_i$ the ground truth target values, $\hat{y}_i$ the predicted values and $\bar{y}$ is the mean of the true target.
\end{itemize}
These metrics allow us to assess the performance of active learning methods in terms of both the quantities of interest and the sensitivity indices, as well as the metamodel itself.

 \subsection{Application on test functions}
The acquisition functions are first tested on classical test functions widely used in GSA. In particular, we consider the Ishigami function ($d=3$) and the G–Sobol function ($d=6,15$), for which all Sobol' indices are known analytically, as well as the Hartmann function ($d=4$), often used in optimization benchmarking and for which the indices must be estimated numerically. Finally, we include a fixed path of a Gaussian process ($d=2$), which allows us to verify the behavior of the methods when the GP hypothesis is exactly satisfied. A reminder of the analytical expressions of these functions, together with the distributions of their inputs, is provided in Appendix \ref{appendix:A}. The numerical experiments are grouped into two figures: one for a selection of classical test cases (Figure \ref{fig:toyfunc}) and one for higher-dimensional ones (Figure \ref{fig:othertoyfunctions}). The results for DGSM estimation show that the most accurate estimates are consistently obtained with global gradient-based acquisition functions. Local gradient-based criteria also outperform purely derivative-based approaches, which is expected given that the acquisition functions are specifically designed to exploit gradient information. For the total Sobol' indices estimation from the metamodel, the proposed methods achieve accuracy comparable to, and in several cases better than, competing strategies. Nevertheless, some variability across replications is observed, with randomized enrichment occasionally yielding the best performance. This variability decreases as the input dimension increases, which can be explained by the fact that the considered functions remain relatively easy to approximate with a sampling budget of order $15d$ using Sobol' sequences. In this regime, the subsequent pick-freeze estimation of Sobol' indices is already reliable, and differences between acquisition strategies naturally diminish.\\

A complementary visualization is provided in Figure \ref{fig:Toy2D}, illustrating the spatial distribution of enrichment points for a two-dimensional toy function. Sobol’ sampling leads to a uniform coverage of the input domain. In contrast, the \texttt{GradMaxVar} strategy exhibits very poor exploration: points tend to accumulate on the boundary of the domain and are often repeatedly selected, since the criterion does not reduce gradient uncertainty once a point has been sampled. The \texttt{GradVarRed} strategy provides a more balanced behavior, exploring the domain more effectively and frequently adding points in pairs or small clusters, reflecting the benefit of neighboring evaluations for gradient estimation. Finally, the \texttt{GlobalGradVarRed} strategy concentrates points in regions where gradients are most informative while preserving a good level of global exploration, achieving the best compromise between exploration and exploitation among the presented method.\\

On the Ishigami function, additional diagnostic plots (Figures \ref{fig:parallelPlotIshi}–\ref{fig:Q2partialIshi}) further highlight the differences between strategies. The parallel coordinate plot (Figure \ref{fig:parallelPlotIshi}) shows that \texttt{GlobalGradVarRed} naturally aligns its sampling with the dominant direction (the second input), whereas Sobol’ points remain more diffuse. The histograms of partial derivatives (Figure \ref{fig:HistPartialIshi}) confirm that \texttt{GlobalGradVarRed} explores regions where $\frac{\partial f}{\partial x_2}$ is large while allocating fewer points to irrelevant regions (e.g., for $\frac{\partial f}{\partial x_3}$). Finally, the evolution of the $Q^2$ coefficient for predicting partial derivatives (Figure \ref{fig:Q2partialIshi}) shows substantially faster improvement for \texttt{GlobalGradVarRed}, especially on the influential derivative $\frac{\partial f}{\partial x_2}$, while both methods perform modestly on less influential derivatives, a behavior consistent with their sensitivity indices. These observations demonstrate that the gradient-based global criterion effectively identifies and prioritizes directions of high importance.\\

Finally, although the acquisition functions focus only on squared gradients and partial derivatives, the overall surrogate accuracy—as measured by the coefficient of determination—is comparable to, and sometimes better than, state-of-the-art alternatives for a limited evaluation budget. Taken together, these experiments indicate that incorporating gradient information (without gradient observations of the \textit{true} model) within a global criterion yields improved performance over existing methods.

\begin{figure}
    \centering
    \includegraphics[width=\linewidth]{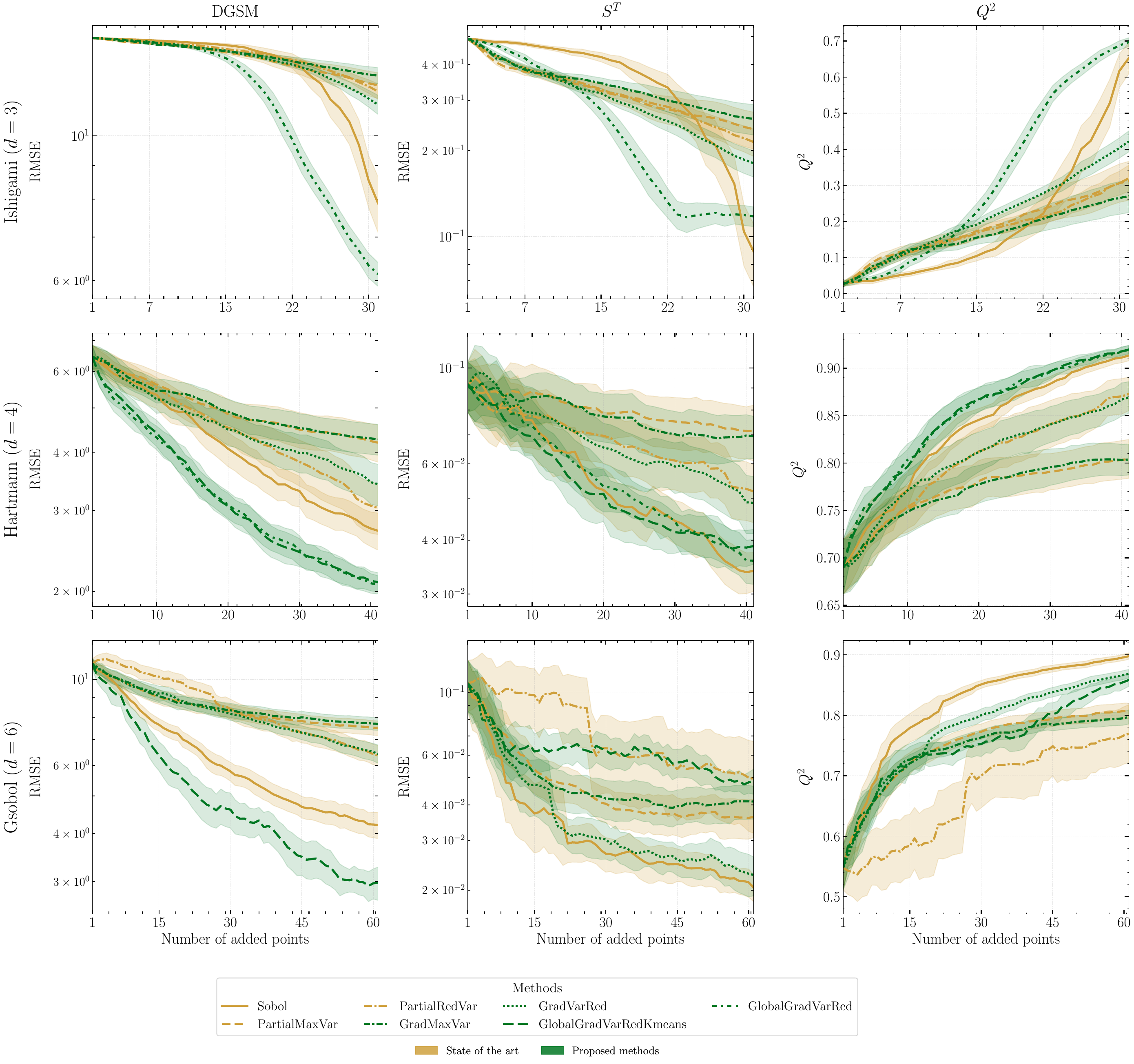}
    \caption{Active learning strategies on a selection of classical test functions.}
    \label{fig:toyfunc}
\end{figure}
\begin{figure}[ht]
    \centering
    \includegraphics[width=0.8\linewidth]{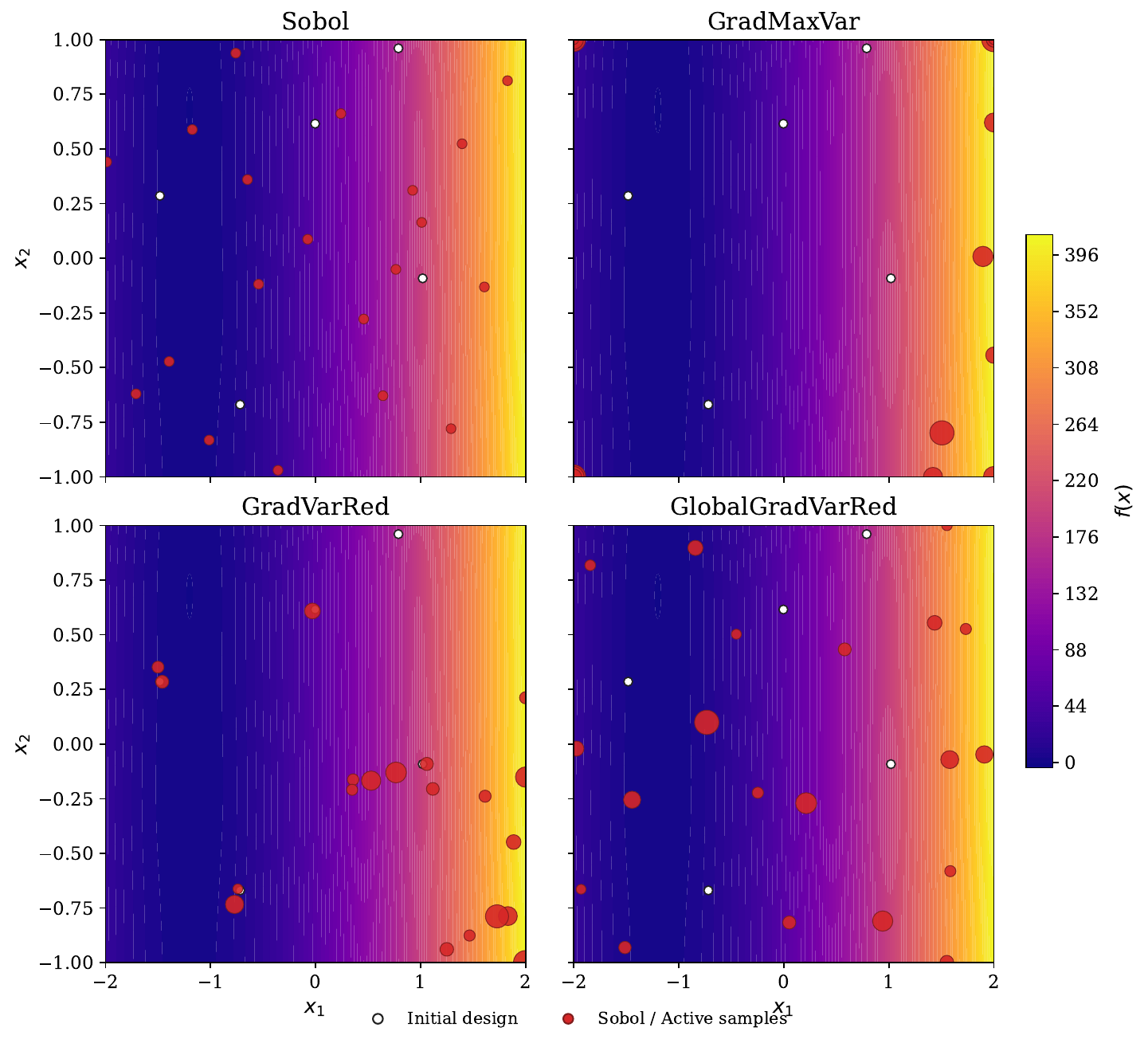}
    \caption{Spatial distribution of enrichment points on a two-dimensional toy function for the \texttt{LocalGradVarRed} (top right), \texttt{LocalGradMaxVar} (bottom left), \texttt{GlobalGradVarRed} (bottom right) acquisition functions  and for Sobol' random sampling (top left).}
    \label{fig:Toy2D}
\end{figure}
\subsection{Application on BUVARD-MES}
\subsubsection*{Presentation of the model}
BUVARD-MES is a decision-support tool for the design of vegetative filter strips (VFSs) aimed at reducing runoff-driven transfers of sediments and pesticides in agricultural catchments (Figure \ref{fig:fig10}). The tool and its modeling foundations are described in detail in \cite{Carluer2017,veillon2022buvardmes,QLHS}. In this work, we consider the Morcille catchment (Figure \ref{fig:fig11}) digital twin as a realistic test bed for our active-learning methodology: it provides spatially distributed simulations of field-to-VFS transfers over a vineyard watershed in the Beaujolais region (France), supported by long-term monitoring data \cite{Morcille_datapaper,QLHS}.
\begin{figure}[htbp]
    \centering
    \includegraphics[width=0.8\textwidth]{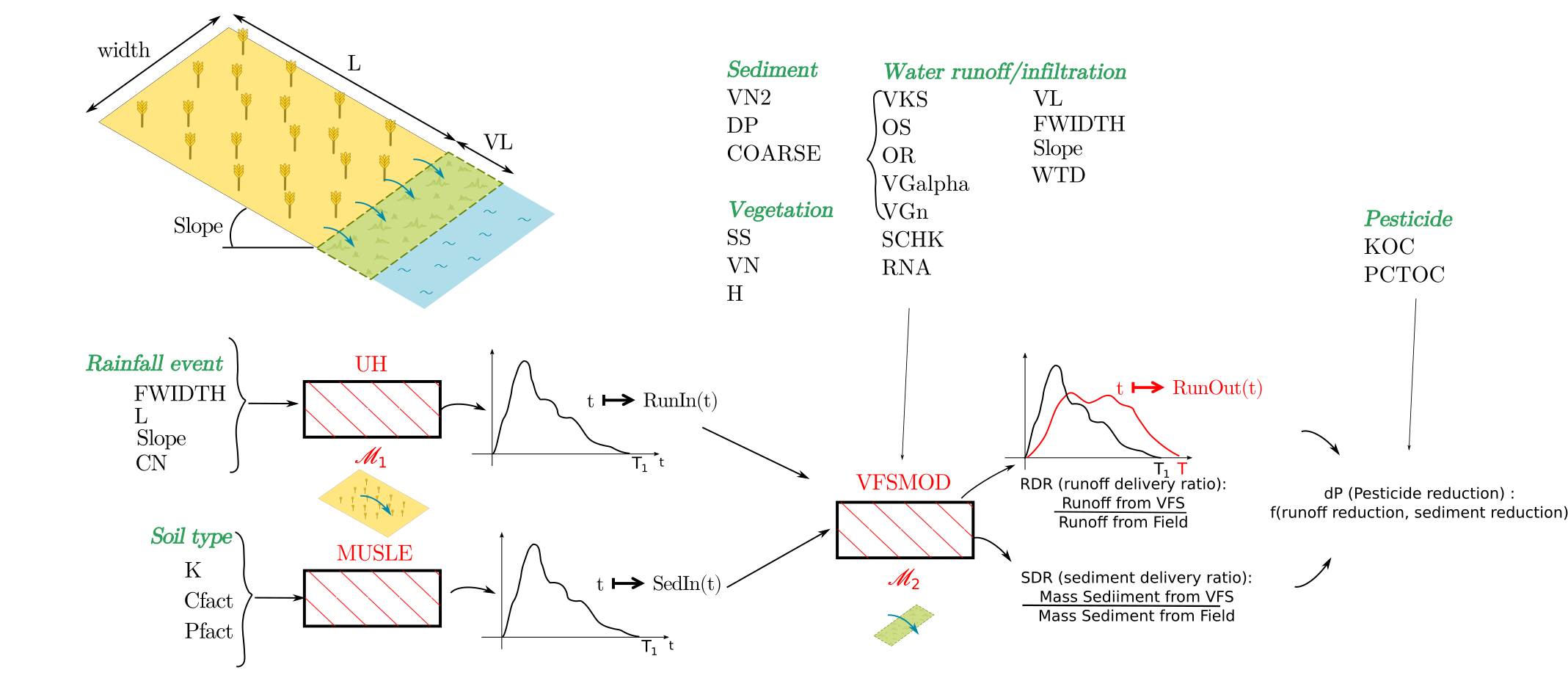}
    \caption{BUVARD\_MES model and its sub-models, with inputs for climate, soil, vegetation properties of the fields and VFSs. The group of (Van Genuchten) dependent parameters is indicated by a brace.}
    \label{fig:fig10}
\end{figure}
\begin{figure}[htbp]
    \centering
    \includegraphics[width=0.8\textwidth]{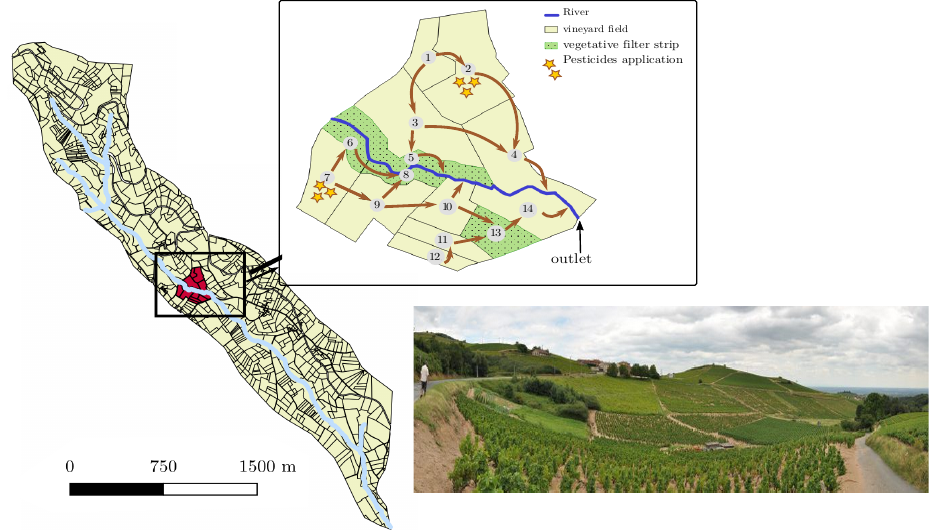}
    \caption{The Morcille catchment (left and bottom) and its digital twin, in the Beaujolais vineyard region (France). Example of surface properties extracted from the virtual catchment (top).}
    \label{fig:fig11}
\end{figure}
\subsubsection*{Adapted acquisition function for constrained input space}
The input space of BUVARD-MES is composed of several variables characterizing the crop field, the vegetative filter strip, and the pesticides. These variables are described in Table \ref{tab:variables}. In this work, we consider a simplified version of BUVARD-MES, involving a reduced set of input variables, in order to focus on the proposed active learning methodology.
\begin{table}[h!]
\centering
\begin{tabular}{l l l}
\toprule
\textbf{Variable} & \textbf{Description} & \textbf{Distribution} \\
\midrule
$\text{Area\_CA}$ & Contributive Area & KDE\\
$\text{Width\_VFS}$ & Width of Vegetative Filter Strip & KDE \\
$\text{WTD\_VFS}$ & Water-table depth & $\mathcal{U}(0.4, 2)$ \\
$\text{VL\_VFS}$ & Length of Vegetative Filter Strip & $\mathcal{U}(5, 30)$ \\
$\text{CN}$ & Curve number & $\mathcal{U}(69,89)$ \\
$\alpha_{VG}$ & Inverse of the air entry suction 
    & \multirow{2}{*}{\centering Mixture $(\alpha_{VG}, n_{VG})$} \\
$n_{VG}$ & Measure of the pore-size distribution & \\
\bottomrule
\end{tabular}
\caption{Description of the input variables of the simplified BUVARD-MES model.}
\label{tab:variables}
\end{table}
Some variables are represented using kernel density estimators fitted to field observations. Moreover, the pair $(\alpha_{VG}, n_{VG})$ has a complex joint distribution driven by the underlying soil physics: not all combinations of these parameters are physically admissible, and many non-realistic scenarios are ruled out. As a consequence, the joint distribution exhibits extended regions of zero density, and $(\alpha_{VG}, n_{VG})$ effectively lives on a constrained support (further structured by the soil type, which induces dependence between the two variables). As a result, the methodology developed above cannot be applied directly to this group of inputs. In particular, the acquisition function is optimized with a multi-start L-BFGS procedure over a hypercube domain, which may propose candidate points in regions where the joint density of $(\alpha_{VG}, n_{VG})$ is zero, i.e., outside the empirical support. Optimizing the acquisition function over such a hypercube is therefore inappropriate: it ignores the zero-probability regions of the admissible input space and can lead to infeasible enrichment points. To address this issue, we build on classical ideas from constrained Bayesian optimization \cite{constBO} and smooth support approximation. We adapt the acquisition function by introducing a differentiable penalization term that down-weights the acquisition value outside the joint support of $(\alpha_{VG}, n_{VG})$. The resulting penalized \texttt{GlobalGradVarRed} acquisition function is
\[
\alpha(\mathbf{x}) = \rho(\mathbf{x})\alpha_{\mathrm{GlobalGradVarRed}}(\mathbf{x}),
\]
where $\rho : \mathbb{R}^d \to [0,1]$ is a smooth penalty. Let $\mathbf{x} = (x_1,\dots,x_d)\in\mathbb{R}^d$ and let $
\mathbf{z} = (x_{i_1},\dots,x_{i_s})^\top\in\mathbb{R}^s$
denote a group of $s$ dependent variables such that $\mathbf{x} = (\mathbf{z, {}_{\sim}\mathbf{z})}$. We approximate the joint distribution of $\mathbf{z}$ by a $k$-component Gaussian mixture model with means $\bm{\mu}_i \in \mathbb{R}^s$ and covariances $\Sigma_i \in \mathcal{S}^{++}_{s}(\mathbb{R})$, $i = 1,\dots, k$.
For each component, we recall the Mahalanobis distance
\[
d_i(\mathbf{z}) = 
\sqrt{(\mathbf{z}-\bm{\mu}_i)^\top \Sigma_i^{-1} (\mathbf{z}-\bm{\mu}_i)},
\quad i=1,\dots,k,
\]
and choose a radius $R_i>0$. Then, the empirical support of the dependent group is approximated by the union of ellipsoids
\[
\mathcal{S} = \bigcup_{i=1}^k \left\{ \mathbf{z} \in \mathbb{R}^s \mid d_i(\mathbf{z}) \leq R_i \right\}
\]
Such unions of overlapping  ellipsoids are a parametric analog of the robust support approximations used in geometric inference, where compact sets are reconstructed via unions of balls or ellipsoids \cite{claire_brecheteau}. To construct a smooth surrogate of the indicator of $\mathcal{S}$, we rely on a standard smooth relaxation based on log-sum-exp and sigmoid functions. First, we define the signed distances $
s_i(\mathbf{z}) = d_i(\mathbf{z}) - R_i
$ and compute the soft minimum using the log-sum-exp operator \cite{boyd2004convex}
\[
S(\mathbf{z}) 
= 
\log \left(
\sum_{i=1}^k 
\exp\bigl(- s_i(\mathbf{z})\bigr)
\right)
\]
Finally, the smooth penalty is defined as,
\[
\rho(\mathbf{x})
= 
\sigma\left(a(b - S(\mathbf{z}))\right),
\]
where $\sigma : \mathbb{R} \to [0,1]$ is the sigmoid function, $a>0$ controls the transition sharpness, and $b\geq 0$ is a dilation factor of the feasible set. This parametrization provides flexibility in controlling rejection regions while preserving differentiability for L-BFGS. \\

Under the same numerical settings as in Section \ref{section:4}, we applied our acquisition function to the BUVARD-MES model. Due to the increased computational cost of this model, we report results over 10 independent repetitions (instead of 30 in Section \ref{section:4}). In Figure \ref{fig:plot_VG}, we can analyze where the points are added during the active learning loop for the Van Genuchten group $(\alpha_{VG}, n_{VG})$. Also, we see the penalty boundary that set the acquisition function to zero if outside the feasible region. Figure \ref{fig:resBuvard} shows the results of estimating the DGSM and total Sobol indices, as well as $Q^2$, highlighting the effectiveness of the acquisition method for building a good metamodel in a low-budget context, and demonstrating the quality of the sensitivity indices estimates. 
\begin{figure}
    \centering
    \includegraphics[width=0.6\linewidth]{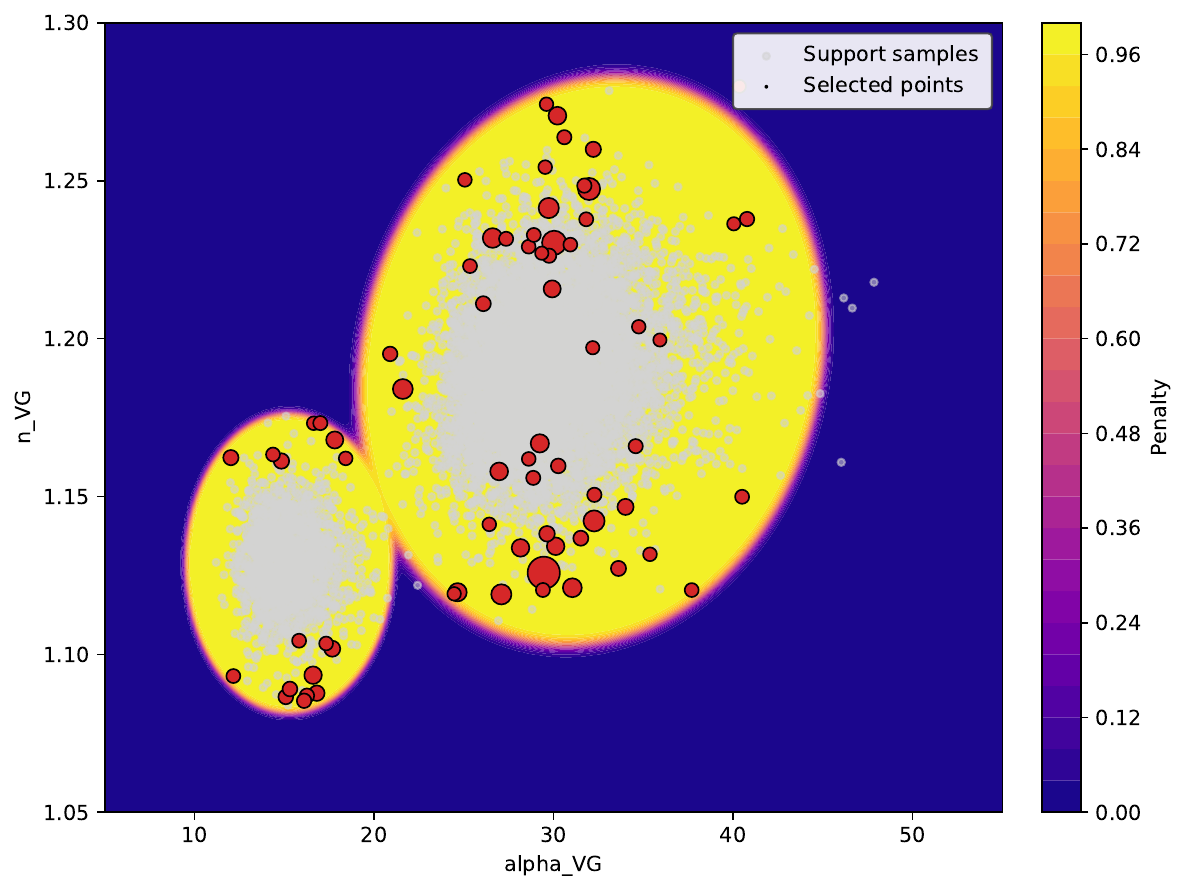}
    \caption{Localization of points of the penalized \texttt{GlobalGradVarRed} for the Van Genuchten group $(\alpha_{VG}, n_{VG})$ and penalty region for one active learning run. The width of the red points is proportional to the acquisition function value.}
    \label{fig:plot_VG}
\end{figure}
\begin{figure}
    \centering
    \includegraphics[width=\linewidth]{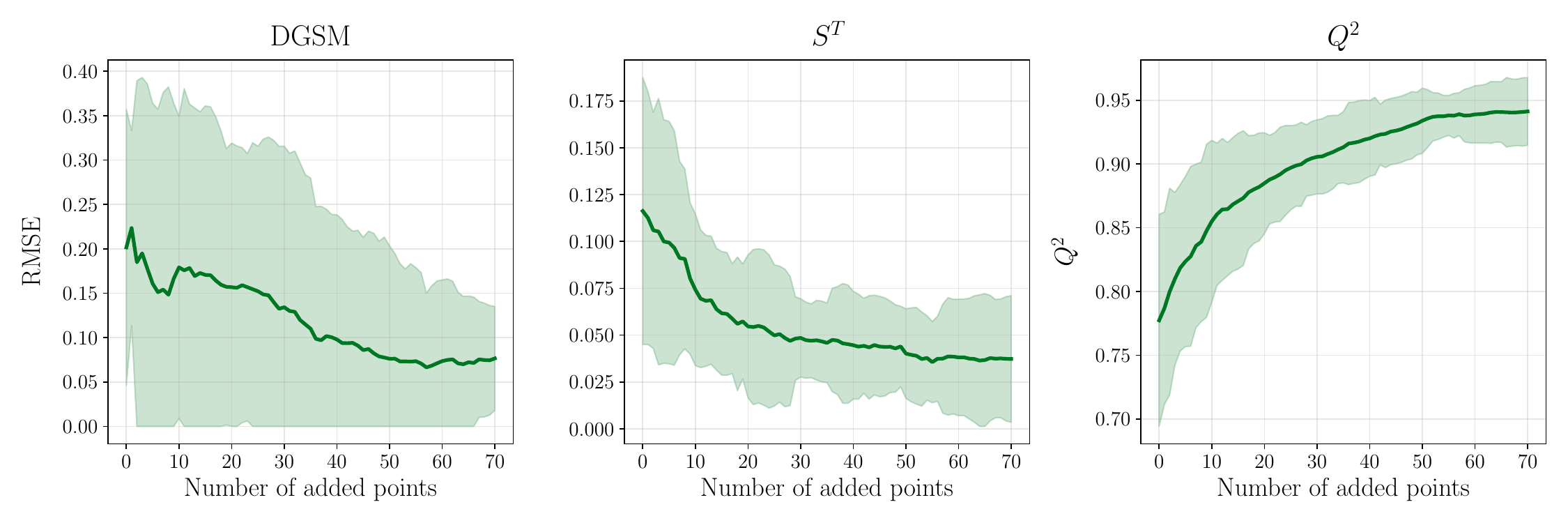}
    \caption{Active learning strategy on the simplified BUVARD-MES model.}
    \label{fig:resBuvard}
\end{figure}

\section{Conclusion}
This work proposes a gradient-based active learning framework for Gaussian process surrogates dedicated to global sensitivity analysis. By introducing acquisition functions that exploit the full gradient structure of the GP, the proposed approach preserves the correlation between derivative components and their joint impact on the response surface. This leads to metamodels that are both predictive and well suited for estimating DGSM and Sobol indices, outperforming existing gradient-based strategies that rely on partial-derivative information only. The methodology is further extended to settings where the input space is constrained by nonlinear physical or probabilistic conditions, such as regions of zero density. These extensions demonstrate the flexibility of the approach in realistic scenarios. The application to the BUVARD-MES model illustrates the practical benefits of gradient-based active learning in reducing computational cost while maintaining accurate sensitivity estimates. Future work may focus on improving acquisition function optimization under nonlinear constraints and on extending the proposed framework to more complex inputs and outputs, including functional data and spatiotemporal responses.

\section*{Acknowledgments}
This research was carried out with the support of the \href{https://doi.org/10.5281/zenodo.6581216}{CIROQUO} Applied Mathematics consortium, which brings together partners from both industry and academia to develop advanced methods for computational experimentation. The research is part of Water4All's AQUIGROW project, which aims to enhance the resilience of groundwater services under increased drought risk. The project has received funding from the European Union's Horizon Europe Program under Grant Agreement 101060874. This work was granted access to the HPC resources of PMCS2I (Pôle de Modélisation et de Calcul en Sciences de l'Ingénieur de l'Information) of the École Centrale de Lyon, Écully, France.
\section*{Data Availability Statement}
The research code associated with this article is available on \href{https://github.com/gulambert/ActiveDGSM}{GitHub}.

\printbibliography

@article{sobol_sensitivity_1993,
	title = {Sensitivity {Estimates} for {Nonlinear} {Mathematical} {Models}},
	author = {Sobol, I.},
	year = {1993},
}

@inproceedings{Belakaria,
 author = {Belakaria, S. and Letham, B. and Doppa, J. and Engelhardt, B. and Ermon, S. and Bakshy, E.},
 booktitle = {Advances in Neural Information Processing Systems},
 title = {Active Learning for Derivative-Based Global Sensitivity Analysis with Gaussian Processes},
 year = {2024},
doi = {10.52202/079017-1706}
}

@article{Lozzo,
author = {De Lozzo, M. and Marrel, A.},
title = {Estimation of the Derivative-Based Global Sensitivity Measures Using a Gaussian Process Metamodel},
journal = {SIAM/ASA Journal on Uncertainty Quantification},
volume = {4},
number = {1},
pages = {708-738},
year = {2016},
doi = {10.1137/15M1013377}
}

@book{gsa,
  title={Global Sensitivity Analysis: The Primer},
  author={Saltelli, A. and Ratto, M. and Andres, T. and Campolongo, F. and Cariboni, J. and Gatelli, D. and Saisana, M. and Tarantola, S.},
  isbn={9780470725177},
  year={2008},
  publisher={Wiley},
doi ={10.1002/9780470725184}
}

@book{gp,
    author = {Rasmussen, C E and Williams, C K. I.},
    title = {Gaussian Processes for Machine Learning},
    publisher = {The MIT Press},
    year = {2005},
    month = {11},
    isbn = {9780262256834},
    doi = {10.7551/mitpress/3206.001.0001},
}

@article{sobol1967distribution,
  title={On the distribution of points in a cube and the approximate evaluation of integrals},
  author={Sobol, I.M},
  journal={Zhurnal Vychislitel'noi Matematiki i Matematicheskoi Fiziki},
  volume={7},
  number={4},
doi = {10.1016/0041-5553(67)90144-9},
  pages={784--802},
  year={1967},
  publisher={Russian Academy of Sciences, Branch of Mathematical Sciences}
}

@article{DGSM,
title = {Derivative based global sensitivity measures},
journal = {Procedia - Social and Behavioral Sciences},
volume = {2},
number = {6},
pages = {7745-7746},
year = {2010},
note = {Sixth International Conference on Sensitivity Analysis of Model Output},
issn = {1877-0428},
doi = {10.1016/j.sbspro.2010.05.208},
author = {I.M. Sobol and S. Kucherenko},
}

@inproceedings{veillon2022buvardmes,
  author = {Veillon, F.  and Carluer, N. and Lauvernet, C. and Rabotin, M.},
  title = {{BUVARD-MES : Un outil en ligne pour dimensionner les zones tampons enherbées afin de limiter les transferts de pesticides vers les eaux de surface}},
  booktitle = {50e congrès du Groupe Français des Pesticides},
  eventdate = {2022-05-18/2022-05-20},
  venue = {NAMUR},
year = {2022},
  note={in french}

}

@inproceedings{gpytorch,
  title={GPyTorch: Blackbox Matrix-Matrix Gaussian Process Inference with GPU Acceleration},
  author={Gardner, J. R. and Pleiss, G. and Bindel, D. and Weinberger, K. Q. and Wilson, A. G.},
  booktitle={Advances in Neural Information Processing Systems},
  year={2018}
}

@inproceedings{botorch,
  title={{BoTorch: A Framework for Efficient Monte-Carlo Bayesian Optimization}},
  author={Balandat, Maximilian and Karrer, Brian and Jiang, Daniel R. and Daulton, Samuel and Letham, Benjamin and Wilson, Andrew Gordon and Bakshy, Eytan},
  booktitle = {Advances in Neural Information Processing Systems 33},
  year={2020}}

@Article{Carluer2017,
  Title                    = {Defining context-specific scenarios to design vegetated buffer zones that limit pesticide transfer via surface runoff },
  Author                   = {Carluer, N. and Lauvernet, C. and Noll, D. and Muñoz-Carpena,R.},
  Journal                  = {Science of the Total Environment },
  Year                     = {2017},
  Pages                    = {701 - 712},
  Volume                   = {575},
  Doi                      = {10.1016/j.scitotenv.2016.09.105},
  ISSN                     = {0048-9697}
}

@Article{Morcille_datapaper,
  author   = {Gouy, V. and Liger, L. and Ahrouch, S. and Bonnineau, C. and Carluer, N. and Chaumot, A. and Coquery, M. and Dabrin, A. and Margoum, C. and Pesce, S.},
  title    = {Ardières-Morcille in the Beaujolais, France: A research catchment dedicated to study of the transport and impacts of diffuse agricultural pollution in rivers},
  journal  = {Hydrological Processes},
  year     = {2021},
  volume   = {35},
  number   = {10},
  pages    = {e14384},
  doi      = {10.1002/hyp.14384},
}

@article{wycoff,
author = {N. Wycoff and M. Binois and S. M. Wild},
title = {Sequential Learning of Active Subspaces},
journal = {Journal of Computational and Graphical Statistics},
volume = {30},
number = {4},
pages = {1224--1237},
year = {2021},
publisher = {ASA Website},
doi = {10.1080/10618600.2021.1874962},
}

@book{boyd2004convex,
  title={Convex Optimization},
  author={Boyd, S. and Vandenberghe, L.},
  isbn={9781107394001},
  year={2004},
  publisher={Cambridge University Press}
}

@article{claire_brecheteau,
	author = {Brecheteau, C. and Genetay, E. and Mathieu, T. and Saumard, A},
	title = {Topics in robust statistical learning},
	DOI= "10.1051/proc/202374119",
	journal = {ESAIM: ProcS},
	year = 2023,
	volume = 74,
	pages = "119-136",
}

@article{picheny,
author = {Picheny, Victor and Wagner, Tobias and Ginsbourger, David},
title = {A benchmark of kriging-based infill criteria for noisy optimization},
year = {2013},
publisher = {Springer-Verlag},
address = {Berlin, Heidelberg},
volume = {48},
number = {3},
issn = {1615-147X},
doi = {10.1007/s00158-013-0919-4},
journal = {Struct. Multidiscip. Optim.},
pages = {607–626},
numpages = {20},
}

@article{QLHS,
author = {Lambert, G. and Helbert, C. and Lauvernet, C.},
title = {Quantization-Based Latin Hypercube Sampling for Dependent Inputs With an Application to Sensitivity Analysis of Environmental Models},
journal = {Applied Stochastic Models in Business and Industry},
volume = {41},
number = {3},
pages = {e2899},
doi = {10.1002/asmb.2899},
year = {2025}
}

@article{mckay_comparison_1979,
    title = {A {Comparison} of {Three} {Methods} for {Selecting} {Values} of {Input} {Variables} in the {Analysis} of {Output} from a {Computer} {Code}},
    volume = {21},
    issn = {00401706},
    doi = {10.2307/1268522},
    number = {2},
    journal = {Technometrics},
    author = {McKay, M. D. and Beckman, R. J. and Conover, W. J.},
    month = may,
    year = {1979},
    pages = {239},
}

@InProceedings{hsic,
author="Gretton, A.
and Bousquet, O.
and Smola, A.
and Sch{\"o}lkopf, B.",
title="Measuring Statistical Dependence with Hilbert-Schmidt Norms",
booktitle="Algorithmic Learning Theory",
year="2005",
publisher="Springer Berlin Heidelberg",
address="Berlin, Heidelberg",
pages="63--77",
}

@article{IMSE2,
    author = {Picheny, V. and Ginsbourger, D. and Roustant, O. and Haftka, R. T. and Kim, N-H.},
    title = {Adaptive Designs of Experiments for Accurate Approximation of a Target Region},
    journal = {Journal of Mechanical Design},
    volume = {132},
    number = {7},
    pages = {071008},
    year = {2010},
    month = {06},
    issn = {1050-0472},
    doi = {10.1115/1.4001873},
}

@article{IMSE3,
title = {Sequential design strategies for mean response surface metamodeling via stochastic kriging with adaptive exploration and exploitation},
journal = {European Journal of Operational Research},
volume = {262},
number = {2},
pages = {575-585},
year = {2017},
issn = {0377-2217},
doi = {10.1016/j.ejor.2017.03.042},
author = {Xi Chen and Qiang Zhou},
}

@article{IMSE1,
author = {Gauthier, B. and Pronzato, L.},
title = {Spectral Approximation of the IMSE Criterion for Optimal Designs in Kernel-Based Interpolation Models},
journal = {SIAM/ASA Journal on Uncertainty Quantification},
volume = {2},
number = {1},
pages = {805-825},
year = {2014},
doi = {10.1137/130928534}}

@inproceedings{gratiet,
  author    = {Le Gratiet, L. and Couplet, M. and Iooss, B. and Pronzato, L.},
  title     = {Planification d’exp{\'e}riences s{\'e}quentielle pour l’analyse de sensibilit{\'e}},
  booktitle = {46\`emes Journ{\'e}es de Statistique},
  address   = {Rennes, France},
  pages     = {60},
  year      = {2014}
}

@article{lamboni,
title = {Derivative-based global sensitivity measures: General links with Sobol’ indices and numerical tests},
journal = {Mathematics and Computers in Simulation},
volume = {87},
pages = {45-54},
year = {2013},
issn = {0378-4754},
doi = {10.1016/j.matcom.2013.02.002},
author = {M. Lamboni and B. Iooss and A.-L. Popelin and F. Gamboa},
}

@Inbook{GSA_review,
author="Iooss, B.
and Lema{\^i}tre, P.",
editor="Dellino, G.
and Meloni, C.",
title="A Review on Global Sensitivity Analysis Methods",
bookTitle="Uncertainty Management in Simulation-Optimization of Complex Systems: Algorithms and Applications",
year="2015",
publisher="Springer US",
address="Boston, MA",
pages="101--122",
isbn="978-1-4899-7547-8",
doi="10.1007/978-1-4899-7547-8_5",
}

@article{PCE_GSA,
title = {Global sensitivity analysis using polynomial chaos expansions},
journal = {Reliability Engineering \& System Safety},
volume = {93},
number = {7},
pages = {964-979},
year = {2008},
note = {Bayesian Networks in Dependability},
issn = {0951-8320},
doi = {10.1016/j.ress.2007.04.002},
author = {B. Sudret},
}

@article{MARREL2009742,
title = {Calculations of Sobol indices for the Gaussian process metamodel},
journal = {Reliability Engineering \& System Safety},
volume = {94},
number = {3},
pages = {742-751},
year = {2009},
issn = {0951-8320},
doi = {10.1016/j.ress.2008.07.008},
author = {A. Marrel and B. Iooss and B. Laurent and O. Roustant},
}

@article{blockVecchia,
author = {Qilong Pan and Sameh Abdulah and Marc G. Genton and Ying Sun},
title = {Block Vecchia Approximation for Scalable and Efficient Gaussian Process Computations},
journal = {Technometrics},
volume = {67},
number = {3},
pages = {546--558},
year = {2025},
publisher = {ASA Website},
doi = {10.1080/00401706.2025.2475784},
}

@article{Vecchia,
author = {M. Katzfuss and J. Guinness},
title = {{A General Framework for Vecchia Approximations of Gaussian Processes}},
volume = {36},
journal = {Statistical Science},
number = {1},
publisher = {Institute of Mathematical Statistics},
pages = {124 -- 141},
year = {2021},
doi = {10.1214/19-STS755},
}

@article{constantine,
author = {Constantine, P. G. and Dow, E. and Wang, Q.},
title = {Active Subspace Methods in Theory and Practice: Applications to Kriging Surfaces},
journal = {SIAM Journal on Scientific Computing},
volume = {36},
number = {4},
pages = {A1500-A1524},
year = {2014},
doi = {10.1137/130916138},
}

@inproceedings{iooss_dgsm,
author = {Iooss, B. and Popelin, A.-L. and Blatman, G. and Ciric, C. and Gamboa, F. and Lacaze, S. and Lamboni, M.},
year = {2012},
month = {06},
pages = {},
title = {Some new insights in derivative-based global sensitivity measures}
}

@inproceedings{Snoek2012PracticalBO,
  title={Practical Bayesian Optimization of Machine Learning Algorithms},
  author={J. Snoek and H. Larochelle and R. P. Adams},
  booktitle={Neural Information Processing Systems},
  year={2012},
}

@Inbook{Ginsbourger2010,
author="Ginsbourger, D.
and {Le Riche}, R.
and Carraro, L.",
title="Kriging Is Well-Suited to Parallelize Optimization",
bookTitle="Computational Intelligence in Expensive Optimization Problems",
year="2010",
publisher="Springer Berlin Heidelberg",
address="Berlin, Heidelberg",
pages="131--162",
isbn="978-3-642-10701-6",
doi="10.1007/978-3-642-10701-6_6",
}

@inproceedings{Ginsbourger2008,
  title={A Multi-points Criterion for Deterministic Parallel Global Optimization based on Gaussian Processes},
  author={D. Ginsbourger and R. {Le Riche} and L. Carraro},
  year={2008},
}

@article{kmeansconstrained,
title = {A review on declarative approaches for constrained clustering},
journal = {International Journal of Approximate Reasoning},
volume = {171},
pages = {109135},
year = {2024},
note = {Synergies between Machine Learning and Reasoning},
issn = {0888-613X},
doi = {10.1016/j.ijar.2024.109135},
author = {Thi-Bich-Hanh Dao and Christel Vrain},

}

@software{kmeansconstrained_software,
  author = {Levy-Kramer, Josh},
  month = apr,
  title = {{k-means-constrained}},
  url = {https://github.com/joshlk/k-means-constrained},
  year = {2018}
}

@Inbook{blockGershgorin,
author="Varga, R. S.",
title="Ger{\v{s}}gorin-Type Theorems for Partitioned Matrices",
bookTitle="Ger{\v{s}}gorin and His Circles",
year="2004",
publisher="Springer Berlin Heidelberg",
address="Berlin, Heidelberg",
pages="155--187",
isbn="978-3-642-17798-9",
doi="10.1007/978-3-642-17798-9_6",
}

@book{basicstrends,
author = {Da Veiga, S. and Gamboa, F. and Iooss, B. and Prieur, C.},
title = {Basics and Trends in Sensitivity Analysis},
publisher = {Society for Industrial and Applied Mathematics},
year = {2021},
doi = {10.1137/1.9781611976694},
address = {Philadelphia, PA},
}

@book{niederreiter1992random,
author = {Niederreiter, Harald},
title = {Random Number Generation and Quasi-Monte Carlo Methods},
publisher = {Society for Industrial and Applied Mathematics},
year = {1992},
doi = {10.1137/1.9781611970081},
address = {},
edition   = {},
}

@article{SOBOL200780,
title = {Quasi-random points keep their distance},
journal = {Mathematics and Computers in Simulation},
volume = {75},
number = {3},
pages = {80-86},
year = {2007},
issn = {0378-4754},
doi = {10.1016/j.matcom.2006.09.004},
author = {I.M. Sobol and B.V. Shukhman},
}

@article{bect_sequential_2012,
	title = {Sequential design of computer experiments for the estimation of a probability of failure},
	volume = {22},
	issn = {1573-1375},
	doi = {10.1007/s11222-011-9241-4},
	pages = {773--793},
	number = {3},
	journaltitle = {Statistics and Computing},
	shortjournal = {Statistics and Computing},
	author = {Bect, Julien and Ginsbourger, David and Li, Ling and Picheny, Victor and Vazquez, Emmanuel},
	date = {2012-05-01},
}

@software{codeGitHub,
  author  = {Lambert, Guerlain},
  title   = {Python code for ``Gradient-based Active Learning with Gaussian Processes for
Global Sensitivity Analysis''},
  year    = {2025},
  url     = {https://github.com/gulambert/ActiveDGSM},
}

@article{blockGershgorin2,
author = {David G. Feingold and Richard S. Varga},
title = {{Block diagonally dominant matrices and generalizations of the Gerschgorin circle theorem.}},
volume = {12},
journal = {Pacific Journal of Mathematics},
number = {4},
publisher = {Pacific Journal of Mathematics, A Non-profit Corporation},
pages = {1241 -- 1250},
year = {1962},
}

@book{horn2012matrix,
  title={Matrix Analysis},
  author={Horn, R.A. and Johnson, C.R.},
  isbn={9781139788885},
  year={2012},
  publisher={Cambridge University Press}
}

@ARTICLE{constBO,
  author={Amini, Sasan and Vannieuwenhuyse, Inneke and Morales-Hernández, Alejandro},
  journal={IEEE Access}, 
  title={Constrained Bayesian Optimization: A Review}, 
  year={2025},
  volume={13},
  number={},
  pages={1581-1593},
  doi={10.1109/ACCESS.2024.3522876}}
\appendix
\section{Gradient of Matérn kernels derivation}\label{appendix:kernel}
In this section, we recall the classical derivations for the ARD Matérn--$5/2$ kernel that are used to implement the joint posterior distribution of the gradient of a GP. To this end, the expressions are written in matrix form as much as possible. For a more detailed derivation, we refer to the derivations provided in the GPyTorch PR that introduced \texttt{Matern52KernelGrad}.\footnote{\url{https://github.com/cornellius-gp/gpytorch/pull/2512}}. For two inputs $x, x' \in \mathbb{R}^d$, define
\[
\delta = x - x', \qquad
\Lambda = \mathrm{diag}(\theta_1,\dots,\theta_d), \qquad
\Lambda^{-2} = \mathrm{diag}(\theta_1^{-2},\dots,\theta_d^{-2}).
\]
Define the distance
\[
q(x,x') = \delta^{\top}\Lambda^{-2}\delta
= \sum_{i=1}^{d}\frac{\delta_i^{2}}{\theta_i^{2}},
\qquad
r(x,x') = \sqrt{q(x,x')}.
\]
The ARD Matérn--$5/2$ kernel is
\begin{equation*}
k(x,x')=\sigma^{2}f(r)e^{-\sqrt{5}r},
\quad
f(r)=1+\sqrt{5}r+\frac{5}{3}r^{2}.
\end{equation*}

\subsection{Gradient of the Matérn  kernel}

Fix $j\in\{1,\dots,d\}$. Since $q(x,x')=\sum_{i=1}^d \delta_i^2/\theta_i^2$ and
$\delta_j=x_j-x'_j$, we have
\[
\frac{\partial q}{\partial x_j}=2\frac{\delta_j}{\theta_j^{2}},
\qquad
\frac{\partial q}{\partial x'_j}=-2\frac{\delta_j}{\theta_j^{2}}.
\]
Moreover, for all $(x,x')\in\left(\mathbb{R}^d\right)^2$ such that $r(x,x')>0$,
\[
\frac{\partial r}{\partial x_j}
=\frac{1}{2}q^{-1/2}\frac{\partial q}{\partial x_j}
=\frac{\delta_j}{r\theta_j^{2}},
\qquad
\frac{\partial r}{\partial x'_j}
=-\frac{\delta_j}{r\theta_j^{2}}.
\]
In particular, for $r>0$,
\[
\nabla_x r=\frac{1}{r}\Lambda^{-2}\delta,
\quad
\nabla_{x'} r=-\frac{1}{r}\Lambda^{-2}\delta.
\]
Set for $r\geq0$, $g(r)=f(r)e^{-\sqrt{5}r}$. Since $f$ is differentiable on $\mathbb{R}^+$, then
$f'(r)=\sqrt{5}+\frac{10}{3}r$, we obtain
\[
g'(r)=\bigl(f'(r)-\sqrt5f(r)\bigr)e^{-\sqrt5 r}.
\]
A direct simplification yields, for all $r\ge 0$,
\[
f'(r)-\sqrt5 f(r)
=\left(\sqrt5+\frac{10}{3}r\right)-\sqrt5\left(1+\sqrt5 r+\frac{5}{3}r^2\right)
= -\frac{5}{3}r(1+\sqrt5 r),
\]
hence
\[
g'(r)=-\frac{5}{3}r(1+\sqrt5 r)e^{-\sqrt5 r}.
\]
Since $k(x,x')=\sigma^2 g(r)$, we finally derive,
\[
\frac{\partial k}{\partial r}
=
\sigma^2 g'(r)
=
-\frac{5\sigma^2}{3}r(1+\sqrt5 r)e^{-\sqrt5 r}.
\]
Therefore, for all $(x,x') \in (\mathbb{R}^d)^2$ such that $r(x,x')>0$,
\begin{equation}\label{eq:gradx_mat52}
\nabla_x k(x,x')
=
\frac{\partial k}{\partial r}\nabla_x r
=
-\frac{5\sigma^2}{3}(1+\sqrt5 r)e^{-\sqrt5 r}\Lambda^{-2}\delta,
\end{equation}
and similarly $\nabla_{x'}k(x,x')=-\nabla_x k(x,x')$.

\subsection{Hessian of the Matérn kernel}

Differentiating $\partial k/\partial r$ once more gives, for all $r\ge 0$,
\begin{equation}\label{eq:krr_mat52}
\frac{\partial^{2} k}{\partial r^{2}}
=
-\frac{5\sigma^{2}}{3}e^{-\sqrt{5}r}\bigl(1+\sqrt{5}r-5r^{2}\bigr).
\end{equation}
Then, let $i,j\in\{1,\dots,d\}$ and suppose $r(x,x')>0$. Using the identities
\[
\frac{\partial r}{\partial x_i}=\frac{\delta_i}{r\theta_i^{2}},
\qquad
\frac{\partial^{2} r}{\partial x_i\partial x_j}
=\frac{\delta_{ij}}{r\theta_i^{2}}
-\frac{\delta_i\delta_j}{r^{3}\theta_i^{2}\theta_j^{2}},
\]
we can write
\[
\frac{\partial^{2}k}{\partial x_i\partial x_j}
=
\frac{\partial^{2}k}{\partial r^{2}}
\frac{\partial r}{\partial x_i}\frac{\partial r}{\partial x_j}
+
\frac{\partial k}{\partial r}
\frac{\partial^{2} r}{\partial x_i\partial x_j}.
\]
Plugging the expressions of $\partial k/\partial r$ and $\partial^2 k/\partial r^2$
and grouping the terms in $\delta_i\delta_j$ gives
\[
\frac{1}{r^2}\frac{\partial^2 k}{\partial r^2}
-\frac{1}{r^3}\frac{\partial k}{\partial r}
=
\frac{25\sigma^2}{3}e^{-\sqrt5 r},
\]
so we obtain, for all $r>0$,
\begin{equation}\label{eq:hess_component_mat52}
\frac{\partial^{2}k}{\partial x_i\partial x_j}
=
\frac{5\sigma^{2}}{3}e^{-\sqrt{5}r}
\left[
5\frac{\delta_i\delta_j}{\theta_i^{2}\theta_j^{2}}
-(1+\sqrt{5}r)\frac{\delta_{ij}}{\theta_i^{2}}
\right].
\end{equation}
Equivalently, in matrix form (for $r>0$),
\begin{equation}\label{eq:hess_mat52}
\nabla_{x}\nabla_{x}k(x,x')
=
\frac{5\sigma^{2}}{3}e^{-\sqrt{5}r}
\Bigl[
5\Lambda^{-2}\delta\delta^{\top}\Lambda^{-2}
-(1+\sqrt{5}r)\Lambda^{-2}
\Bigr].
\end{equation}

Moreover, since $k(x,x')$ depends on $(x,x')$ only through $\delta=x-x'$,
we have $\nabla_{x'}=-\nabla_x$, then for all $r>0$,
\[
\nabla_{x}\nabla_{x'}k(x,x')
=
-\nabla_{x}\nabla_{x}k(x,x').
\]

Finally, the closed forms
\eqref{eq:gradx_mat52}--\eqref{eq:hess_mat52} admit continuous extensions at $r=0$. In particular, at $x=x'$
\[
\nabla_{x}\nabla_{x}k(x,x)
=
-\frac{5\sigma^{2}}{3}\Lambda^{-2}, \quad \nabla_{x}\nabla_{x'}k(x,x)
=
\frac{5\sigma^{2}}{3}\Lambda^{-2}
\]
\section{Proof of Proposition \ref{prop:chunked}}
\label{appendix:proof}
Recall that
$\Sigma_\nabla = [\Sigma_\nabla^{(ij)}]_{1\leq i,j\leq C}$ is the exact covariance of
$Z_{\mathbf{X}_s}$ ordered by clusters, and that
$\widetilde{\Sigma}_\nabla = \mathrm{diag}(\Sigma_\nabla^{(11)},\dots,\Sigma_\nabla^{(CC)})$
is its block-diagonal approximation. Define the residual 
\[
  E
  =
  \Sigma_\nabla - \widetilde{\Sigma}_\nabla
  =
  \begin{pmatrix}
    0 & \Sigma_\nabla^{(12)} & \cdots & \Sigma_\nabla^{(1C)} \\
    \Sigma_\nabla^{(21)} & 0 & \cdots & \Sigma_\nabla^{(2C)} \\
    \vdots & \vdots & \ddots & \vdots \\
    \Sigma_\nabla^{(C1)} & \Sigma_\nabla^{(C2)} & \cdots & 0
  \end{pmatrix}.
\]
By construction, $E$ has zero diagonal blocks and collects all cross-cluster
covariances.

Let
\[
  \widetilde{V}
  =
  2 \Tr\bigl(\widetilde{\Sigma}_\nabla^2\bigr)
  + 4 \bm{\mu}_\nabla^\top \widetilde{\Sigma}_\nabla \bm{\mu}_\nabla
\]
denote the chunked variance. Using the definitions of $V$ and $\widetilde{V}$,
and the triangle inequality, we obtain
\begin{equation}
  \label{eq:V_diff_laprems}
  \bigl| V - \widetilde{V}\bigr|
  \le
  2\bigl| \Tr(\Sigma_\nabla^2) - \Tr(\widetilde{\Sigma}_\nabla^2)\bigr|
  +
  4\bigl| \bm{\mu}_\nabla^\top \Sigma_\nabla \bm{\mu}_\nabla
         - \bm{\mu}_\nabla^\top \widetilde{\Sigma}_\nabla \bm{\mu}_\nabla\bigr|.
\end{equation}
We first show that
\[
  \Tr(\Sigma_\nabla^2) - \Tr(\widetilde{\Sigma}_\nabla^2)
  = \Tr(E^2) = \|E\|_F^2.
\]
Using $\Sigma_\nabla = \widetilde{\Sigma}_\nabla + E$,
\begin{align*}
  \Tr(\Sigma_\nabla^2)
  &= \Tr\bigl((\widetilde{\Sigma}_\nabla + E)^2\bigr)\\
  &= \Tr\bigl(\widetilde{\Sigma}_\nabla^2\bigr)
     + 2\Tr\bigl(\widetilde{\Sigma}_\nabla E\bigr)
     + \Tr(E^2).
\end{align*}
The matrix $\widetilde{\Sigma}_\nabla$ is block-diagonal with cluster-blocks
$\Sigma_\nabla^{(ii)}$ on the diagonal, while $E$ has zero diagonal blocks. Therefore,
$\widetilde{\Sigma}_\nabla E$ has zero diagonal blocks, and its trace vanishes:
\[
  \Tr\bigl(\widetilde{\Sigma}_\nabla E\bigr)
  = \sum_{i=1}^C \Tr\bigl(\Sigma_\nabla^{(ii)} E_{ii}\bigr)
  = 0.
\]
Thus
\[
  \Tr(\Sigma_\nabla^2) - \Tr(\widetilde{\Sigma}_\nabla^2)
  = \Tr(E^2).
\]
Since $E$ is symmetric, $\Tr(E^2) = \|E\|_F^2$, which gives the first term in the
error bound. Then, we write
\begin{align*}
  \bm{\mu}_\nabla^\top \Sigma_\nabla \bm{\mu}_\nabla
  - \bm{\mu}_\nabla^\top \widetilde{\Sigma}_\nabla \bm{\mu}_\nabla
  &= \bm{\mu}_\nabla^\top E \bm{\mu}_\nabla\\
  &= (E\bm{\mu}_\nabla)^\top \bm{\mu}_\nabla.
\end{align*}
By the Cauchy--Schwarz inequality and the definition of the spectral norm,
\[
  \bigl|\bm{\mu}_\nabla^\top E \bm{\mu}_\nabla\bigr|
  \le
  \|E\bm{\mu}_\nabla\|_2\|\bm{\mu}_\nabla\|_2
  \le
  \|E\|_2\|\bm{\mu}_\nabla\|_2^2.
\]
Substituting these two bounds in \eqref{eq:V_diff_laprems}, we obtain
\begin{equation}
  \label{eq:ineg_V}
  \bigl| V - \widetilde{V}\bigr|
  \le
  2\|E\|_F^2
  + 4\|\bm{\mu}_\nabla\|_2^2 \|E\|_2,
\end{equation}
which is the first statement of Proposition \ref{prop:chunked}.
We now derive the bounds on $\|E\|_F$ and $\|E\|_2$.

\subsection{Bounding the Hessian}
\label{appendix:hessian_bound}

Define the following quantities:
\[
  A = \frac{5\sigma^2}{3}, \quad
  \theta_\text{min} = \min_{1 \leq i \leq d } \theta_i, \quad
  L^2 = \sum_{i=1}^d \theta_i^{-4}, \quad
  \delta = x-x',\quad \delta_i = x_i - x_i', \quad \delta_{ab} =
  \begin{cases}
    1,& a=b,\\
    0,& a\neq b.
  \end{cases}
\]
Denote by $H(x, x') \in \mathcal{M}_{d}(\mathbb{R})$ the matrix with components
$H_{ab}(x,x') = \partial_{x_a}\partial_{x_b'}k(x,x')$.

\begin{lemma}\label{lemma:1}
    For any $(x, x')\in \left(\mathbb{R}^d\right)^2$ and $r = r(x,x')$,
    \[
      \Vert H(x, x') \Vert_F^2\leq h(r),
    \]
    where, for any $r \geq 0$,
    \[
      h(r) =
      2A^2\exp(-2\sqrt{5}r)\left(
        25\frac{r^4}{\theta_\text{min}^4}
        + (1+\sqrt{5}r)^2L^2
      \right).
    \]
\end{lemma}

\begin{proof}
    Let $(x,x') \in (\mathbb{R}^d)^2$ and $1 \leq a,b \leq d$, and write
    $H_{ab} := H_{ab}(x,x')$.
    From the derivations in Appendix \ref{appendix:kernel}, we have
    \begin{align*}
        \vert H_{ab}\vert
        &=\frac{5\sigma^2}{3\theta_a^2}\exp(-\sqrt{5}r)
          \left\vert \frac{5(x_a - x_a')(x_b-x_b')}{\theta_b^2}
                     - \left(1+\sqrt{5}r\right)\delta_{ab}\right\vert\\
        &= \frac{A}{\theta_a^2}\exp(-\sqrt{5}r)
          \left\vert \frac{5\delta_a \delta_b}{\theta_b^2}
                     - \left(1+\sqrt{5}r\right)\delta_{ab}\right\vert\\
        &\leq \frac{A}{\theta_a^2}\exp(-\sqrt{5}r)
          \left(
            \frac{5\vert \delta_a\delta_b\vert}{\theta_b^2}
            + \left(1+\sqrt{5}r\right)\delta_{ab}
          \right).
    \end{align*}
    Define
    \[
      u_{ab} := \frac{\vert \delta_a\delta_b\vert}{\theta_a^2\theta_b^2}, \qquad
      v_{ab} := \frac{1 + \sqrt{5}r}{\theta_a^2}\delta_{ab},
    \]
    so that
    \[
      |H_{ab}|
      \leq A\exp(-\sqrt{5}r)\left(5u_{ab} + v_{ab}\right).
    \]
    Squaring and using $(a+b)^2 \leq 2(a^2+b^2)$, we obtain
    \begin{align*}
        H_{ab}^2
        &\leq
        2A^2 \exp(-2\sqrt{5}r)\left(25u_{ab}^2 + v_{ab}^2\right).
    \end{align*}
    Summing over $a$ and $b$ gives
    \begin{align*}
      \Vert H(x,x') \Vert_F^2
      &= \sum_{1\leq a,b\leq d} H_{ab}^2\\
      &\leq 2 A^2\exp(-2\sqrt{5}r)
        \left(
          25\sum_{1\leq a,b\leq d}u_{ab}^2
          + \sum_{1\leq a,b\leq d}v_{ab}^2
        \right).
    \end{align*}
    For the first sum, note that
    \[
      \sum_{a,b} u_{ab}^2
      = \sum_{a,b}
        \frac{\delta_a^2}{\theta_a^4}
        \frac{\delta_b^2}{\theta_b^4}
      = \left(
          \sum_{a=1}^d
          \frac{\delta_a^2}{\theta_a^4}
        \right)^2.
    \]
    Moreover,
    \[
      \sum_{a=1}^d\frac{\delta_a^2}{\theta_a^4}
      \leq \frac{1}{\theta_\text{min}^2}
          \sum_{a=1}^d\frac{\delta_a^2}{\theta_a^2}
      = \frac{r^2}{\theta_\text{min}^2},
    \]
    so that
    \[
      \sum_{a,b} u_{ab}^2
      \leq \frac{r^4}{\theta_\text{min}^4}.
    \]
    For the second sum, using $\delta_{ab}^2=\delta_{ab}$,
    \begin{align*}
      \sum_{a,b} v_{ab}^2
      &= \sum_{a,b}
         \frac{(1+\sqrt{5}r)^2}{\theta_a^4}\delta_{ab}\\
      &= (1+\sqrt{5}r)^2 \sum_{a=1}^d\theta_a^{-4}
       = (1+\sqrt{5}r)^2 L^2.
    \end{align*}
    Plugging these bounds in the inequality above yields
    \[
      \|H(x,x')\|_F^2
      \le
      2A^2\exp(-2\sqrt{5}r)\left(
        25\frac{r^4}{\theta_\text{min}^4}
        + (1+\sqrt{5}r)^2L^2
      \right)
      = h(r),
    \]
    as claimed.
\end{proof}

\subsection{Bounds on $\Vert E\Vert_F$ and $\Vert E\Vert_2$}
\begin{proposition}\label{prop:frob_bound}
    Let $\Delta_{ij} = \min_{x\in C_{i}, x'\in C_{j}} r(x,x')$ for $i\neq j$
    and $\Delta = \min_{i\neq j} \Delta_{ij}$. Then, for clusters of sizes
    $n_1, \dots, n_C$ with $N = \sum_{i=1}^C n_i$, we have
    \begin{align}
        \Vert E\Vert_F^2
        &\leq \sum_{i\neq j}n_i n_j h(\Delta_{ij})
        \leq \left(N^2 - \sum_{i=1}^C n_i^2\right)h(\Delta),
        \label{eq:E_F_final}\\[0.4em]
        \Vert E\Vert_2
        &\leq B\sqrt{h(\Delta)},
        \label{eq:E_2_final}
    \end{align}
    where $h$ is given in Lemma \ref{lemma:1} and
    \[
      B =
      \begin{cases}
        N\bigl(1-\tfrac{1}{C}\bigr), & \text{for balanced clusters }(n_i=n_j),\\[0.3em]
        \dfrac{N + (C-2)n_\text{max}}{2},
        & \text{for imbalanced clusters, with } n_\text{max} = \max_i n_i.
      \end{cases}
    \]
\end{proposition}

\begin{proof}
    The block matrix $E$ is composed of the off-diagonal cluster-blocks
    $\Sigma_\nabla^{(ij)}$ for $i\neq j$. Thus,
    \[
      \Vert E \Vert_F^2
      = \sum_{i\neq j}\Vert E_{ij}\Vert_F^2
      = \sum_{i\neq j}\sum_{x\in C_{i}}\sum_{x'\in C_j}
        \Vert H(x, x')\Vert_F^2,
    \]
    where $E_{ij} = \Sigma_\nabla^{(ij)}$.
    Applying Lemma \ref{lemma:1} to every pair $(x,x')$ yields
    \[
      \Vert E \Vert_F^2
      \leq \sum_{i\neq j}\sum_{x\in C_{i}}\sum_{x'\in C_j} h(r(x,x')).
    \]
    Since $h$ is non-increasing on $\mathbb{R}^+$ (see Lemma \ref{lem:h_qui_decroit}) and, for each cluster pair
    $i\neq j$,
    \[
      \Delta_{ij}
      = \min_{x\in C_i, x'\in C_j} r(x,x'),
    \]
    we obtain
    \begin{align*}
        \Vert E\Vert_F^2
        &\leq \sum_{i\neq j}\sum_{x\in C_{i}}\sum_{x'\in C_j} h(\Delta_{ij})\\
        &= \sum_{i\neq j}n_i n_j h(\Delta_{ij})\\
        &\leq \left(N^2 - \sum_{i=1}^C n_i^2\right)h(\Delta),
    \end{align*}
    which proves \eqref{eq:E_F_final}.
From Lemma \ref{lem:block_E_spec_ineg},
\[
  \|E\|_2 \leq \max_{1\leq i\leq C}\sum_{j\neq i}\|E_{ij}\|_2.
\]
Since $\|E_{ij}\|_2 \leq \|E_{ij}\|_F$ and
    \[
      \|E_{ij}\|_F^2
      = \sum_{x\in C_i}\sum_{x'\in C_j}\|H(x,x')\|_F^2
      \leq n_i n_j h(\Delta_{ij}),
    \]
    Lemma \ref{lemma:1} implies
    \[
      \|E_{ij}\|_F
      \leq \sqrt{n_i n_j}\sqrt{h(\Delta_{ij})}
      \leq \sqrt{n_i n_j}\sqrt{h(\Delta)}.
    \]
    Hence
    \[
      \|E\|_2
      \le
      \sqrt{h(\Delta)}
      \max_{1\leq i\leq C}
      \sum_{j\neq i} \sqrt{n_i n_j}.
    \]
    We now bound the combinatorial quantity
    \[
      S_i := \sum_{j\neq i} \sqrt{n_i n_j}.
    \]

    If the clusters are balanced, with $n_i = N/C$ for all $i$, then for any $i$
    \[
      S_i = (C-1)\frac{N}{C} = N\left(1-\frac{1}{C}\right),
    \]
    which gives $B = N(1-1/C)$.

    If the clusters are imbalanced, we use the AM--GM inequality:
    for every $j\neq i$,
    \[
      \sqrt{n_i n_j}\leq \frac{1}{2}(n_i + n_j).
    \]
    Summing over $j\neq i$ yields
    \begin{align*}
      \sum_{j\neq i} \sqrt{n_i n_j}
      &\leq \frac{1}{2}\left(\sum_{j\neq i}n_i + \sum_{j\neq i} n_j\right)\\
      &= \frac{1}{2}\left((C-1)n_i + N - n_i\right)
       = \frac{1}{2}\left(N +(C-2)n_i\right)\\
      &\leq \frac{1}{2}\left(N +(C-2)n_\text{max}\right),
    \end{align*}
    where $n_\text{max} = \max_i n_i$. Taking the maximum over $i$
    completes the proof of \eqref{eq:E_2_final}.
\end{proof}

Combining \eqref{eq:ineg_V} with
\eqref{eq:E_F_final} and \eqref{eq:E_2_final} yields the bounds stated in
Proposition \ref{prop:chunked}. Here are two lemmas to justify the monotonicity of $h$ and a spectral norm bound used in the previous proof.
\begin{lemma}[Monotonicity of $h$]\label{lem:h_qui_decroit}
The function $h : \mathbb{R}^+ \to \mathbb{R}$ defined in Lemma \ref{lemma:1} is non-increasing on $\mathbb{R}^+$.
\end{lemma}

\begin{proof}
Since $h$ is a product of a polynomial in $r$ and an exponential function, then $h\in\mathcal{C}^1(\mathbb{R}^+;\mathbb{R)}$. Then, for all $r\in \mathbb{R}^+$,
\begin{align*}
h'(r)&=2A^2 e^{-2\sqrt{5}r}\left(\frac{100r^3}{\theta_\text{min}^4} + 2L^2\sqrt{5}\left(1+\sqrt{5}r\right)  - 2\sqrt{5}\left(\frac{25r^4}{\theta_\text{min}^4 }+ \left(1 + \sqrt{5}r\right)^2L^2\right)\right)\\
& = 2A^2 e^{-2\sqrt{5}r} g(r)
\end{align*}
Since $2A^2 e^{-2\sqrt{5}r}$ is positive for all $r\in\mathbb{R}^+$, it is enough to prove that $g(r)\leq 0$.
Moreover, for all $r\in\mathbb{R}^+$,
\begin{align*}
g(r)
&=
\left(100\frac{r^3}{\theta_{\min}^4}+2\sqrt{5}(1+\sqrt{5}r)L^2\right)-2\sqrt{5}\left(25\frac{r^4}{\theta_{\min}^4}+(1+\sqrt{5}r)^2L^2\right)\\
&=
\frac{50r^3}{\theta_{\min}^4}\bigl(2-\sqrt{5}r\bigr)
-2\sqrt{5}L^2(1+\sqrt{5}r)\bigl((1+\sqrt{5}r)-1\bigr)\\
&=
\frac{50r^3}{\theta_{\min}^4}\bigl(2-\sqrt{5}r\bigr)
-10r(1+\sqrt{5}r)L^2.
\end{align*}
Using $L^2=\sum_{i=1}^d\theta_i^{-4}\geq \theta_{\min}^{-4}$, we obtain for all $r\geq 0$,
\begin{align*}
g(r)
&\leq
\frac{50r^3}{\theta_{\min}^4}\bigl(2-\sqrt{5}r\bigr)
-\frac{10r(1+\sqrt{5}r)}{\theta_{\min}^4}\\
&=
\frac{10r}{\theta_{\min}^4}\Bigl(5r^2(2-\sqrt{5}r)-(1+\sqrt{5}r)\Bigr)
=
\frac{10r}{\theta_{\min}^4}G(r),
\end{align*}
where
\[
G(r)=-5\sqrt{5}r^3+10r^2-\sqrt{5}r-1.
\]
One remark that, $G(0)=-1$ and
\[
G'(r)=-15\sqrt{5}r^2+20r-\sqrt{5}.
\]
The roots of $G'$ are $r=\frac{\sqrt{5}}{15}$ and $r=\frac{\sqrt{5}}{5}$. Moreover,
\[
G\left(\frac{\sqrt{5}}{15}\right)=-\frac{31}{27}<-1,
\quad
G\left(\frac{\sqrt{5}}{5}\right)=-1,
\quad
G(0)=-1,
\]
Since, $G(r) \to -\infty$ as $r \to +\infty$, then the global maximum on $\mathbb{R}^+$ is $\max_{r\geq 0} G(r) = -1$. Hence,  $G(r)\leq -1$ for all $r\geq 0$. Therefore, $
g(r)\leq \frac{10r}{\theta_{\min}^4}\cdot (-1)\leq 0$, which implies $h'(r)\leq 0$.
\end{proof}
\begin{lemma}[Spectral norm]\label{lem:block_E_spec_ineg}
Given the residual $E = \Sigma_\nabla - \widetilde{\Sigma}_\nabla$, then
\[
\|E\|_2 \le \max_{1\le i\le C}\sum_{j\neq i}\|E_{ij}\|_2.
\]
\end{lemma}
\begin{proof}
The following is a classical result on block‑operator norms. It can also be seen as a consequence of Gershgorin theorem for symmetric matrices \cite{horn2012matrix, blockGershgorin, blockGershgorin2}. For completeness, we give a short proof based on induced‑norm estimates. We use the block structure of $E$. For a vector
$v = (v_1,\dots,v_C)$ with $v_i\in\mathbb{R}^{n_i d}$, we have
\[
  (Ev)_i = \sum_{j\neq i} E_{ij} v_j.
\]
Hence, by the triangle inequality and submultiplicativity,
\[
  \|(Ev)_i\|_2
  \leq \sum_{j\neq i}\|E_{ij}\|_2\|v_j\|_2.
\]
Define $\alpha\in\mathbb{R}^C$ by $\alpha_j=\|v_j\|_2$ and the matrix
$A\in\mathbb{R}^{C\times C}$ by $a_{ii}=0$ and $a_{ij}=\|E_{ij}\|_2$ for $i\neq j$.
Then the previous inequality reads $\|(Ev)_i\|_2 \leq (A\alpha)_i$, so
\[
  \|Ev\|_2^2
  = \sum_{i=1}^C \|(Ev)_i\|_2^2
  \leq \sum_{i=1}^C (A\alpha)_i^2
  = \|A\alpha\|_2^2.
\]
Therefore $\|Ev\|_2 \leq \|A\|_2\|\alpha\|_2 = \|A\|_2\|v\|_2$.
Since $E$ is symmetric, $A$ is symmetric. Hence,
\[
  \|A\|_2 \leq \sqrt{\|A\|_1\|A\|_\infty} = \|A\|_\infty
  = \max_{1\leq i\leq C}\sum_{j\neq i}\|E_{ij}\|_2.
\]
Combining the above bounds and taking the supremum over $\|v\|_2=1$ yields
\[
  \|E\|_2 \leq \max_{1\leq i\leq C}\sum_{j\neq i}\|E_{ij}\|_2.
\]
\end{proof}

\section{Implementation details}\label{appendix:A}
\subsection{Implementation detail}
All methods are implemented using GPytorch \cite{gpytorch} and BoTorch \cite{botorch} framework, which benefit from auto-differentiation and available on GitHub \cite{codeGitHub}.
\subsection{Test functions}
For each test function, the input space will be described as a uniform random variable $X\sim \mathcal{U}(\mathcal{X})$, where $\mathcal{X}\subset\mathbb{R}^d$ is a compact set dependent on the function of study.
\subsubsection{Ishigami}
Let $a = 7$ and $b = 0.05$, we define the Ishigami function as :
\[\forall (x_1, x_2, x_3)\in [-\pi,\pi]^3,\quad f(x_1, x_2, x_3) = \sin(x_1) + a \sin(x_2)^2 + b x_3^4 \sin(x_1)\]
\subsubsection{G-Sobol}
In dimension $d \geq 1$, the G-Sobol function is defined as :
\[\forall \mathbf{x}\in [0,1]^d, \quad f(\mathbf{x}) = \prod_{i=1}^{d}\frac{\|4x_i-2\|+a_i}{1+a_i}\]
where $a \in (\mathbb{R}^+)^d$ is a vector to be specified. The first and total order Sobol indices are known is closed form for $A \subset\mathcal{P}(\{1,\dots, d\})$: 

\[S_A = \frac{V_A}{V}\]
with,
\begin{itemize}
    \item $V_A = \prod_{j\in A} V_j \quad\text{and}\quad V_j = \frac{1}{3(1+a_j)^2}$
\end{itemize}
\subsubsection{Hartmann}
For $d=4$, the Hartmann function is defined as \cite{picheny}:
\[\forall \mathbf{x}\in[0,1]^4,\quad f(\mathbf{x}) = \frac{1}{0.839}\left(1.1-\sum_{i=1}^{4} \alpha_i \exp\left( -\sum_{j=1}^{4} A_{ij} (x_j - P_{ij})^2 \right)\right)\]
with constants : 
\[
\boldsymbol{\alpha} = 
\begin{bmatrix}
1.0 & 1.2 & 3.0 & 3.2
\end{bmatrix},
\]
\[
A = 
\begin{bmatrix}
10 & 3 & 17 & 3.5 \\
0.05 & 10 & 17 & 0.1 \\
3 & 3.5 & 1.7 & 10 \\
17 & 8 & 0.05 & 10
\end{bmatrix},
\quad
P = 10^{-4} \cdot
\begin{bmatrix}
1312 & 1696 & 5569 & 124 \\
2329 & 4135 & 8307 & 3736 \\
2348 & 1451 & 3522 & 2883 \\
4047 & 8828 & 8732 & 5743
\end{bmatrix}.
\]
\subsubsection{Morris}

In dimension $d=20$,
\[\forall \mathbf{x}\in[0,1]^d,\quad f(\mathbf{x}) = \sum_{i=1}^{20} \beta_i w_i + \sum_{i<j} \beta_{ij} w_i w_j + \sum_{i<j<k} \beta_{ijk} w_i w_j w_k,\]
where:
\[
w_i =
\begin{cases}
2\left( \dfrac{1.1 x_i}{x_i + 0.1} - 0.5 \right), & \text{if } i \in \{3, 5, 7\} \\
2(x_i - 0.5), & \text{otherwise}
\end{cases}
\]
The coefficients are defined as follows:
\[
\beta_i = 
\begin{cases}
20, & \text{if } i \le 10 \\
(-1)^{i+1}, & \text{if } i > 10
\end{cases}
\]

\[
\beta_{ij} = 
\begin{cases}
-15, & \text{if } i \le 6 \text{ or } j \le 6 \\
(-1)^{i + j + 2}, & \text{otherwise}
\end{cases}
\]

\[
\beta_{ijk} =
\begin{cases}
-10, & \text{if } i \le 5 \text{ or } j \le 5 \text{ or } k \le 5 \\
0, & \text{otherwise}
\end{cases}
\]
\section{Additional numerical experiments}
\begin{figure}
    \centering
    \includegraphics[width=0.7\linewidth]{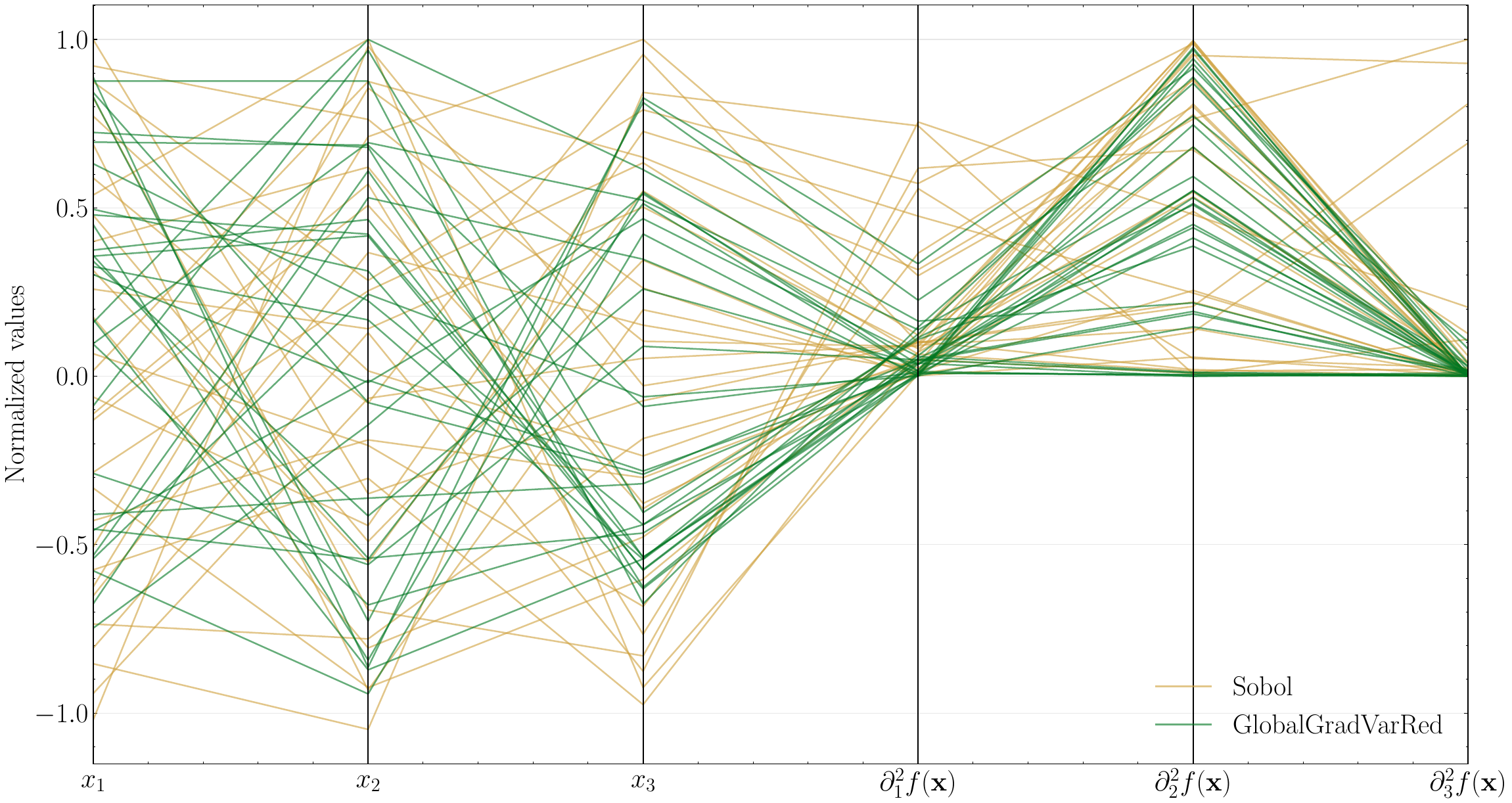}
    \caption{Parallel coordinate plot of Global Variance Reduction and Sobol methods on the Ishigami function. Last three columns correspond to the evaluation of the squared partial derivatives on the active learning points.}
    \label{fig:parallelPlotIshi}
\end{figure}
\begin{figure}
    \centering
    \includegraphics[width=0.7\linewidth]{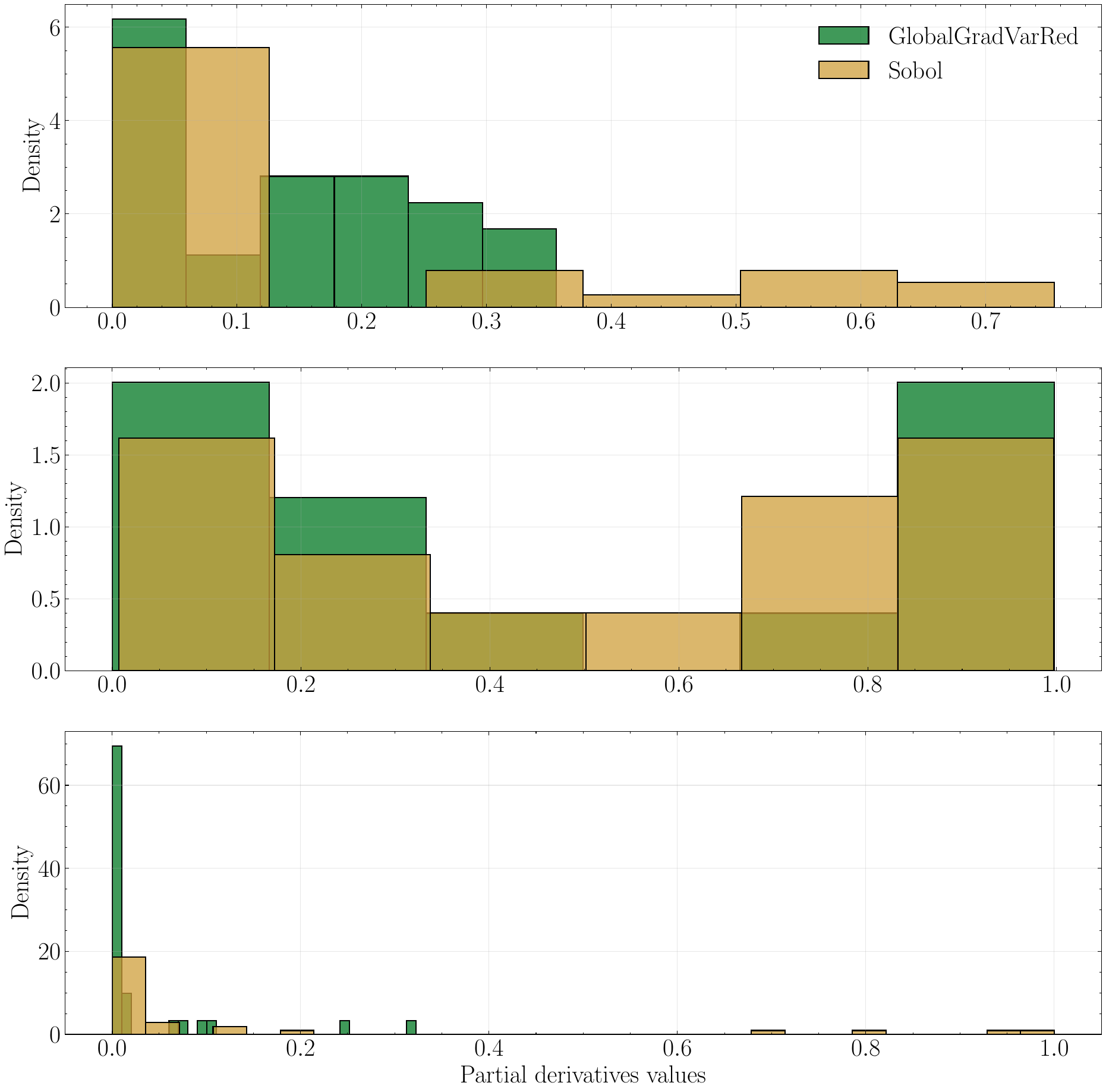}
    \caption{Histograms of repartition of (scaled) partial derivatives values between Sobol and Global Variance Reduction methods.}
    \label{fig:HistPartialIshi}
\end{figure}
\begin{figure}
    \centering
    \includegraphics[width=0.7\linewidth]{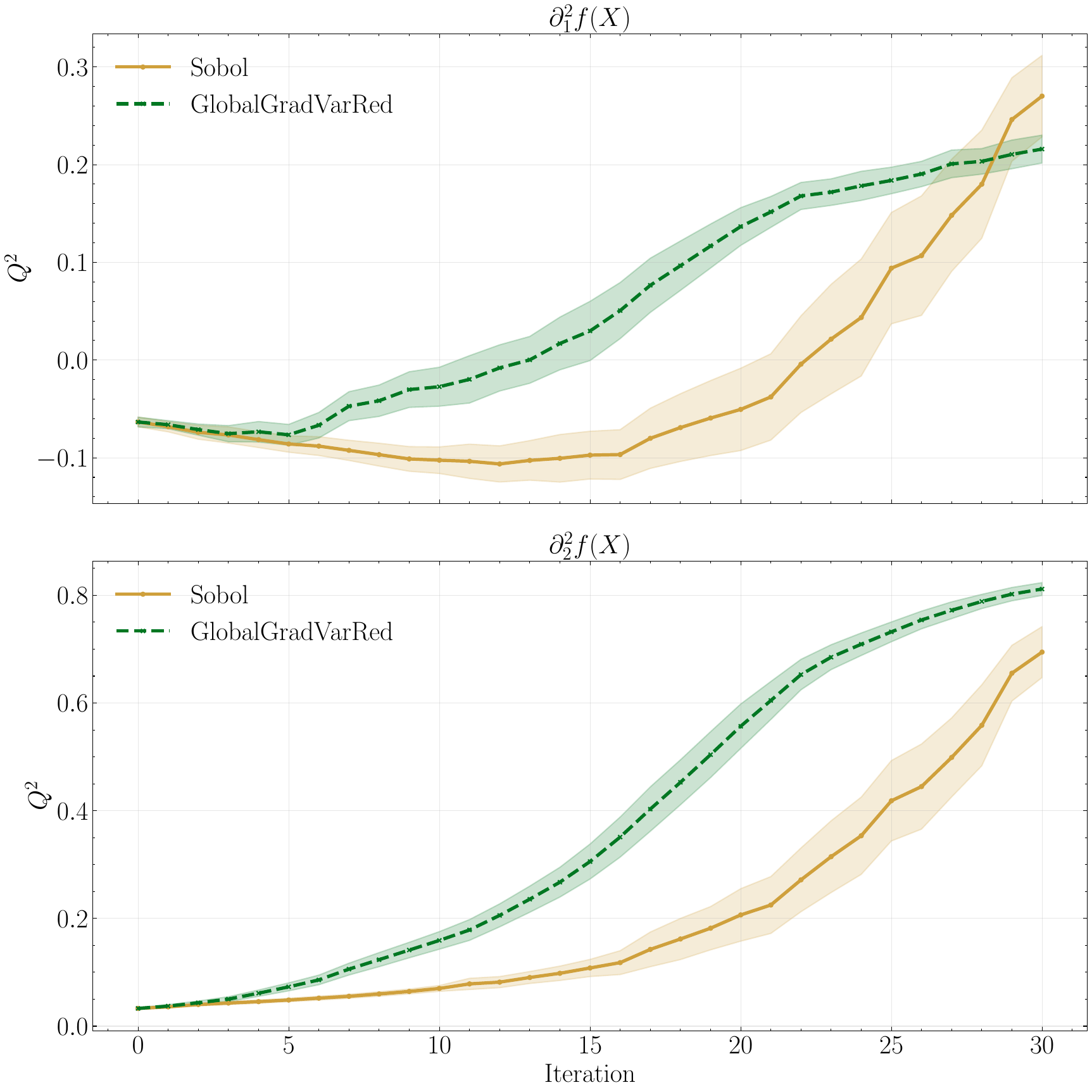}
    \caption{Coefficient of determination $Q^2$ on the first two partial derivatives of Ishigami function for Sobol and Global Variance reductions methods.}
    \label{fig:Q2partialIshi}
\end{figure}
\begin{figure}
    \centering
    \includegraphics[width=\linewidth]{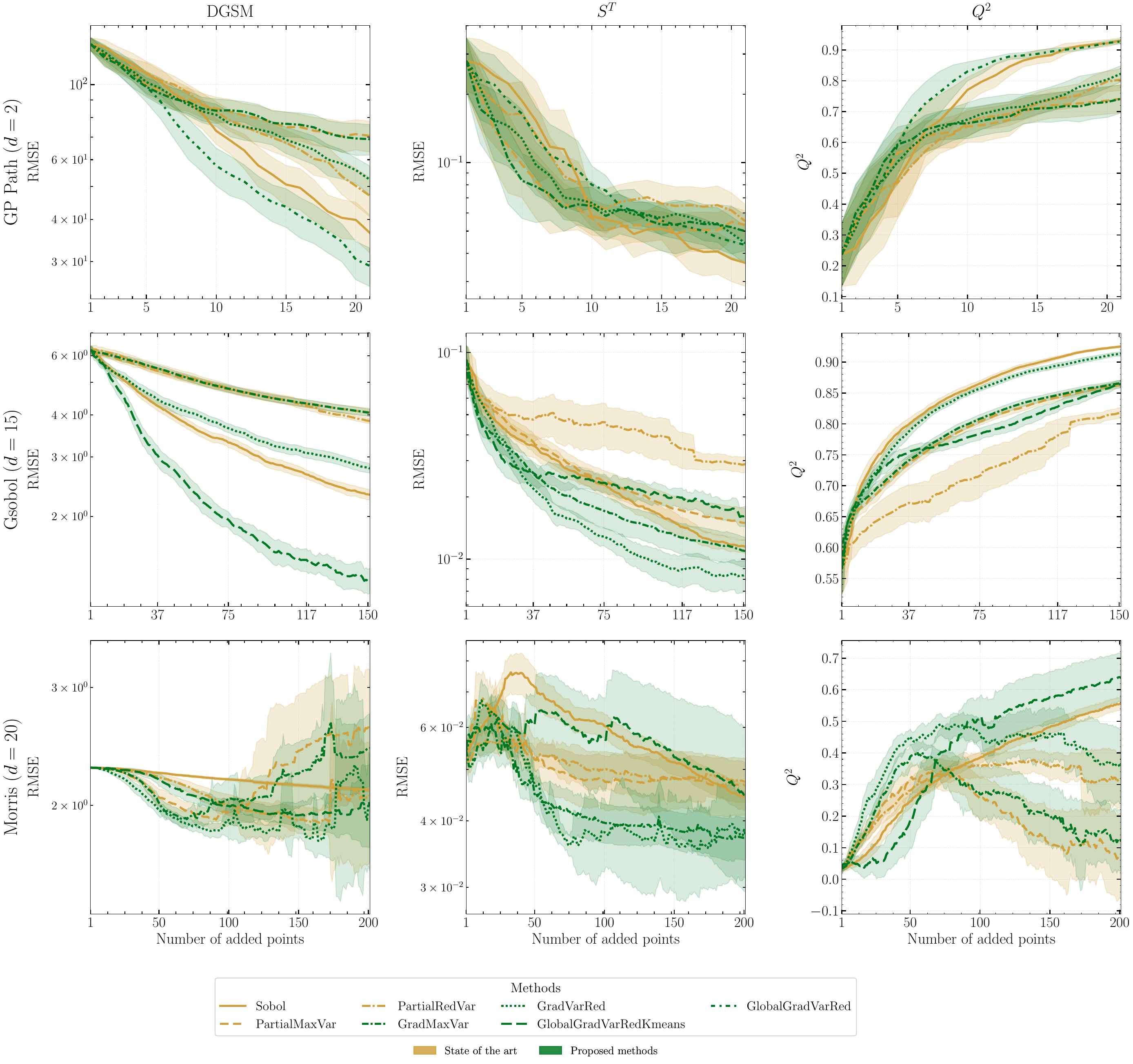}
    \caption{Active learning strategies on a selection of classical test functions.}
    \label{fig:othertoyfunctions}
\end{figure}

\end{document}